\documentclass[letterpaper]{article} 
\usepackage{aaai25}  
\usepackage{times}  
\usepackage{helvet}  
\usepackage{courier}  
\usepackage[hyphens]{url}  
\usepackage{graphicx} 
\usepackage{natbib}  
\usepackage{caption} 
\usepackage{algorithm}
\usepackage{algorithmic}

\usepackage{newfloat}
\usepackage{listings}
\DeclareCaptionStyle{ruled}{labelfont=normalfont,labelsep=colon,strut=off} 
\floatstyle{ruled}
\newfloat{listing}{tb}{lst}{}
\floatname{listing}{Listing}
\usepackage{multirow}
\usepackage{mathrsfs}

\usepackage{amsmath}
\usepackage{amsthm}
\usepackage{amssymb}
\usepackage{physics}
\usepackage{enumerate}

\usepackage{bm}

\newcommand{\fs}{f}
\newcommand{\fh}{\hat{f}}
\newcommand{\fm}{f_{\bf d}}
\newcommand{\fn}{f}
\newcommand{\gs}{g}
\newcommand{\gh}{\hat{g}}
\newcommand{\gm}{g_{\bf d}}
\newcommand{\gn}{g}
\newcommand{\hs}{h}
\newcommand{\hh}{\hat{h}}
\newcommand{\hmm}{h_{\bf d}}
\newcommand{\hn}{h}
\newcommand{\js}{j}
\newcommand{\jh}{\hat{j}}
\newcommand{\jm}{j_{\bf d}}
\newcommand{\jn}{j}

\newcommand{\xs}{x}
\newcommand{\us}{u}
\newcommand{\ys}{y}

\newcommand{\zeros}{0}

\newcommand{\Qs}{Q}
\newcommand{\Rs}{R}
\newcommand{\Ss}{S}

\newcommand{\Ps}{P}
\newcommand{\PsC}{P_{\nabla V }^\mathtt{C}}
\newcommand{\id}{I}

\newcommand{\ls}{\ell}
\newcommand{\Ws}{W}

\newcommand{\Xs}{X}
\newcommand{\Ys}{Y}
\newcommand{\As}{A}
\newcommand{\Bs}{B}
\newcommand{\Cs}{C}

\newtheorem{theorem}{Theorem}
\newtheorem{problem}[theorem]{Problem}

\newtheorem{lemma}[theorem]{Lemma}
\newtheorem{definition}[theorem]{Definition}
\newtheorem{proposition}[theorem]{Proposition}
\newtheorem{corollary}[theorem]{Corollary}

\newcommand{\R}{\mathbb{R}}

\newcommand{\Rnn}{\mathbb{R}_{\geq 0}}

\newcommand{\T}{\mathrm{T}}

\def\inner<#1>{\langle #1 \rangle}

\title{Learning Deep Dissipative Dynamics}
\author {
     Yuji Okamoto\textsuperscript{\rm 1} \equalcontrib,
    Ryosuke Kojima\textsuperscript{\rm 1, 2} \equalcontrib
}
\affiliations {
    \textsuperscript{\rm 1} Kyoto University, Japan\\
    \textsuperscript{\rm 2} RIKEN BDR, Japan\\
    \{okamoto.yuji.2c, kojima.ryosuke.8e\}@kyoto-u.ac.jp
}

\begin{document}

\maketitle

\begin{abstract}
This study challenges strictly guaranteeing ``dissipativity'' of a dynamical system represented by neural networks learned from given time-series data.
Dissipativity is a crucial indicator for dynamical systems that generalizes stability and input-output stability, known to be valid across various systems including robotics, biological systems, and molecular dynamics.
By analytically proving the general solution to the nonlinear Kalman–Yakubovich–Popov (KYP) lemma, which is the necessary and sufficient condition for dissipativity, we propose a differentiable projection that transforms any dynamics represented by neural networks into dissipative ones and a learning method for the transformed dynamics.
Utilizing the generality of dissipativity, our method strictly guarantee  stability, input-output stability, and energy conservation of trained dynamical systems.
Finally, we demonstrate the robustness of our method against out-of-domain input through applications to robotic arms and fluid dynamics.

Code : \url{https://github.com/kojima-r/DeepDissipativeModel}
\end{abstract}
\section{Introduction}

Dissipativity extends the concept of Lyapunov stability to input-output dynamical systems by considering ``energy'' \cite{brogliato2020dissipative}.
In input-output systems, the relationship between the externally supplied energy and the dissipated energy plays an important role.
The theory of dissipativity has wide applications, including electrical circuits \cite{alma9926554206104034}, mechanical systems \cite{hatanaka2015passivity}, and biological systems \cite{goldbeter2018dissipative}.
Considering the inflow, outflow, and storage of energy in a system provides crucial insights for applications such as stability analysis, controller design, and complex interconnected systems.

\begin{figure}[t]
    \centering
     \includegraphics[width=0.9\linewidth]{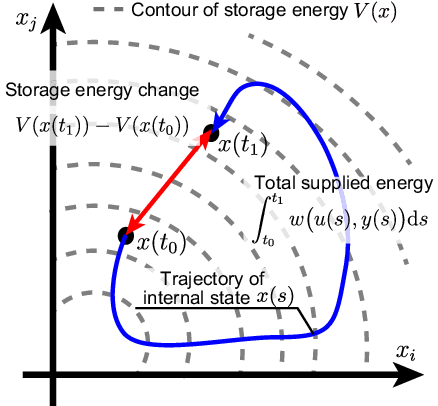}\\
     \caption{Sketch of the dissipativity:
     The red difference in storage energy is less than the total energy supplied along the blue line, which represents the trajectory of the internal state $x(s)$.
     }
    \label{fig:dissipasivity}
\end{figure}

Theoretically, dissipativity in input-output dynamical systems is defined by the time evolution of inputs $u(t)$, outputs $y(t)$, and internal states $x(t)$.
The input-output system is dissipative if the following inequality is satisfies:
\begin{align}
\underbrace{V(\xs(t_1)) - V(\xs(t_0))}_{\text{ Storage Energy Changes}} \leq
\underbrace{\int_{t_0}^{t_1}w\big(\us(s),\ys(s)\big)\mathrm{d}s}_{\text{Total Energy Supplies}}, \label{Eq:dissipativity}
\end{align}
where $V(\xs(t))$ is called the storage energy and  $w(\us(t), \ys(t))$ is called the supply rate.
The left side represents the change in storage energy from the initial state $\xs(t_0)$ to the final state $\xs(t_1)$, while the right side signifies the total supplied energy from $t_0$ to $t_1$ (See Figure~\ref{fig:dissipasivity}).
In Newtonian mechanics, if $V(x)$ is defied as the mechanical energy and $w(u, y)$ is defined as the product of external force \( u \) and velocity \( y \), i.e., \( w(u, y) = uy \), this case corresponds to the principle of energy conservation \cite{Stramigioli2006Geometoric}.
In this context, the integral of the supply rate represents the ``work'' done by the external force $u$.

In this study, we propose an innovative method for learning dynamical systems described by neural networks from time-series data when the system is known a priori to possess dissipativity.
Considering the entire space of the dynamical system described by neural networks, we introduce a transformation of the system by a projection map onto the subspace satisfying dissipativity.
We emphasize that this projection can be applied to the dynamical systems consisting of any differentiable neural networks.
By incorporating this projection into the gradient-based optimization of neural networks, our method allows fitting dissipative dynamics to time-series data.

By configuring the supply rate $w(\cdot,\cdot)$ in the dissipativity, users can design models integrating well-established prior knowledge such as properties of dynamical systems or information from physical systems.
According to the properties of the target dynamical system, such as internal stability, input-output stability, and energy conservation, the supply rate $w(\cdot,\cdot)$ can be constrained.
Within physical systems, the supply rate $w(\cdot,\cdot)$ can be derived from the principle of energy conservation.

Real-world environments often present inputs that vary from the input-output datasets used during model training, for example due to dataset shifts.
Our proposed method guarantees that the trained model strictly satisfies dissipativity for any input time-series data, thereby maintaining robust performance on out-of-domain input.
In this study, we verified the effectiveness of our method, particularly its robustness to out-of-domain input, using both linear and nonlinear systems, including an $n$-link pendulum (a principal model of robotic arms) and the behavior of viscous fluid around a cylinder.

The contributions of this study are as follows:
\begin{enumerate}[(i)]
\item We analytically derived a general solution to the nonlinear KYP lemma and a differentiable projection from the dynamical systems represented by neural networks to a subspace of dissipative ones.
\item We proposed a learning method for strictly dissipative dynamical systems using the above projection.
\item We showed that our learning method generalizes existing methods that guarantee internal stability and input-output stability.
\item We confirmed the effectiveness of our method with three experiments with benchmark data.
\end{enumerate}

\section{Related Work}

\subsubsection{Learning Stable Dynamics.}
In recent years, numerous methods have been proposed for learning models with a priori properties, such as system stability, rather than relying solely on data \cite{blocher2017learning,khansari2011learning,umlauft17a}.
With the advent of deep learning, techniques have been developed to enhance the stability of loss functions compatible with gradient-based learning \cite{richards2018lyapunov}. 

Manek et al. tackled the same internal system but introduced a novel method that guarantees the stability without depending on loss optimization by analytically guaranteeing internal stability \cite{Manek2019}.
This approach was further extended to apply positive invariant sets, such as limit cycles and line attractors, to ensure internal stability \cite{takeishi2021learning}.
Additionally, this analytical approach has been developed for closed-loop systems, ensuring their stability through an SMT solver \cite{Chang2019dynamics}.

Lawrence et al. utilized stochastic dynamical systems, emphasizing internal stability and maintaining it through a loss-based approach \cite{Lawrence2020}.
Similarly, another method has been proposed for state-space models, focusing on input-output stability and ensuring this through projection \cite{kojima2022learning}.

Unlike the above approaches, techniques that imposes constraints on the architecture of neural networks to guarantee energy dissipativity has also been proposed \cite{xu2023learning, sosanya2022dissipative}.

\subsubsection{Hamiltonianian NN.}
Related to the learning of stable systems, Hamiltonian Neural Networks (HNNs) incorporate the principle of energy conservation into their models \cite{greydanus2019hamiltonian}.
Hamiltonian dynamical systems maintain conserved energy, allowing HNNs to learn Hamiltonian functions to predict time evolution.
This method ensures that the model adheres to the conservation of energy law, resulting in physically accurate predictions.

Conversely, some systems exhibit decreasing energy over time without external input, a characteristic known as ``dissipation.''
This property is prevalent in many real-world systems, particularly those involving thermodynamics and friction.
Consequently, methods for learning systems with dissipation from data are gaining interest \cite{drgovna2022dissipative}.

By generalizing dissipation from energy to a broader positive definite function $V$, it can represent a unified concept encompassing input-output stability and Lyapunov stability. 
In this study, we adopted this broader interpretation of dissipativity, allowing us to understand the learning of systems that ensure stability-related properties in a unified framework.
Hereafter, in this paper, we will use the term ``dissipativity'' without distinguishing between ``dissipation'' and ``dissipativity.''

\subsubsection{Neural ODE.}
The state-space dynamic system can be regarded as a differential equation, and our implementation actually uses neural ODE as an internal solver \cite{chen2018neural,chen2019review}.
These techniques have been improved in recent years, including discretization errors and computational complexity.
Although we used an Euler method for simplicity, we can expect that learning efficiency would be further improved by using these techniques.
In this field, methods have been proposed that mainly learn various types of continuous-time dynamics from data.
For example, extended methods for learning stochastic dynamics have been proposed\cite{kidger2020neural,morrill2021neural}

\section{Background}
This study deals with continuous-time state-space models as input-output systems using a nonlinear Lipschitz continuous mapping $\fs(\xs)\in\R^n$ with $f(0) = 0$, continuous mappings $\gs(\xs) \in \R^{n\times m},\hs(\xs)\in \R^l$ with $h(0) = 0$ and $\js(\xs)\in \R^{l\times m}$ formed by neural networks:
\begin{align}
\begin{aligned}
\dot{\xs} &= \fs(\xs) + \gs(\xs)\us , \quad \xs(0) = \xs_0\\
\ys &= \hs(\xs) + \js(\xs)\us
\end{aligned}\label{Eq:main_system}
\end{align}
where the internal state $\xs$, the input $\us$, and the output $\ys$ belong to a signal spaces that maps from time interval $[0,\infty)$ to the $n$, $m$, and $l$ dimensional Euclidean space, respectively.
Here, a tuple $(\fs,\gs,\hs,\js)$ is called a dynamics of the input-output system~(\ref{Eq:main_system}). 
Dissipativity is defined by the supply of energy through the input-output signals $\us,\ys$ and the change in storage energy depending on the internal state $\xs$.
 
\begin{definition}[Dissipativity]
Considering a supply rate $w : \R^m \times  \R^l \rightarrow \R$, there exist a differentiable positive semi-definite  storage function~$V : \R^n \rightarrow \R_{\geq 0}$ such that the input-output system~(\ref{Eq:main_system}) satisfies the dissipative condition (\ref{Eq:dissipativity}), then the system is dissipative.
\end{definition}

Due to the flexible definition of the supply rate~$w(\us, \ys)$, dissipativity can be precisely designed to match the energy conservation law of physical systems, as well as adapt to the properties of dynamical systems, such as internal stability.

For example in Newtonian mechanics, the sum of kinetic and potential energy can be regarded as the storage function $V(x)$.
The supply rate $w(u,y)$ can be determined by the difference between the ``work'' done by external forces and energy dissipativity caused by air resistance or friction.
This work is represented as the integral of the product of velocity and external force.
This energy dissipativity due to friction can be expressed as a quadratic form of velocity.
This supply rate belongs to a quadratic form of external input and observed velocity (see Appendix~\ref{APP:Examples}).

Additionally, dissipativity is defined as an extension of internal stability and input-output stability.
For $w(\us, \ys) = 0$, it corresponds to internal stability, for $w(\us, \ys) = \gamma^2 \|\us\|^2 - \|\ys\|^2$, it corresponds to input-output stability ($\gamma>0$ is the gain of input-output signals).

In general, the supply rate~$w(\us, \ys)$, as described in the above two paragraphs, is represented as a quadratic form of the input and output:
\begin{align}
w(\us,\ys) \triangleq
[\ys^\T,\us^\T] 
\begin{bmatrix}
\Qs&\Ss\\
\Ss^\T&\Rs
\end{bmatrix}
\begin{bmatrix}
\ys\\
\us
\end{bmatrix}\label{Eq:supply_function}
\end{align}
The supply rate parameters $\Qs, \Ss, \Rs$ can be designed in a manner that corresponds to the energy conservation laws of physical systems and the properties of dynamical systems such as internal stability.
When the supply rate can be expressed in this form, there exists a necessary and sufficient condition of dissipative dynamical systems, formulated as the following matrix equation:

\begin{proposition}[Nonlinear KYP lemma {\cite[Theorem 4.101]{brogliato2020dissipative}}]
Consider the input-output system~(\ref{Eq:main_system}) is reachable.
The system~(\ref{Eq:main_system}) is dissipative if and only if there exists $\ls:\R^n \rightarrow  \R^{q}$, $\Ws:\R^n \rightarrow  \R^{q\times m}$ and a differentiable positive semi-definite function $V:\R^n \rightarrow \R_{\geq 0}$ such that 
\begin{align}
     \nabla V^\T(\xs)\fs(\xs) &= \hs^\T(\xs) \Qs \hs(\xs) - \ls^\T (\xs) \ls (\xs),\nonumber\\
    \frac{1}{2} \nabla V^\T(\xs) \gs(x) &=  \hs^\T(\xs)(\Ss + \Qs \js(\xs)) - \ls^\T(\xs) \Ws(\xs),\nonumber\\
    \Ws^\T (\xs) \Ws(\xs) &= \Rs +  \js^\T(\xs) \Ss + \Ss^\T \js(\xs) + \js^\T (\xs) \Qs \js(\xs),\nonumber\\
    &\quad \forall \xs\in \R^n,\label{Eq:KYP_condition}
\end{align}
where the nonlinear dynamical system is reachable if and only if for any $x^*$ there exists $T\geq 0$ and $u$ such that $x(0) = 0$ and $x(T)=  x^*$.
\end{proposition}
\begin{proof}
Appendix~\ref{APP:Proof_of_KYP_lemma}
\end{proof}

The maps $\ls$ and $\Ws$ represent the residuals of the time derivative of the storage function $V$ and the supply rate $w(u,y)$ in the definition of dissipativity.  
The maps $\ls$ corresponds to terms independent of the input $u$, while $\Ws$ corresponds to terms linearly dependent on $u$ (See detail in Appendix~\ref{APP:Proof_of_KYP_lemma}).

Note that the maps $\ls, \Ws$ and $V$ satisfying the condition~(\ref{Eq:KYP_condition}) is not unique.
See Appendix~\ref{APP:Freedom_degree_of_KYP_lemma} for a more detailed discussion about the freedom degree of the dissipativity condition. 

Assuming $\jn \equiv 0$, the conditions for input-output stability can be easily derived from the nonlinear KYP lemma and is known as the Hamilton-Jacobi inequality (Details are provided in the Appendix~\ref{APP:Special_case_KYP_lemma}).
Various dynamical systems except for the field of electronic circuits often lack a direct path $j$ from input to output.

The nonlinear KYP lemma means the existence of a (non-unique) mapping from dissipative dynamics $(f,g,h,j)$ to conditions-satisfying maps $\ls(\xs)$, $\Ws(\xs)$, and  $V(\xs)$.
On the contrary, it has not been demonstrated whether there is a mapping from $(\ls, \Ws, V)$ to the dynamics $(\fs,\gs,\hs,\js)$ satisfying dissipativity.
If it is possible to derive a mapping from $(\ls, \Ws, V)$ to dissipative dynamics $(\fs,\gs,\hs,\js)$, then by managing $(\ls, \Ws, V)$, indirect constraining dissipative dynamics $(\fs,\gs,\hs,\js)$ could become possible.

\section{Method}
\subsection{Projection-based Optimization Framework}
\begin{figure}[t]
    \centering
     \includegraphics[width=1.0\linewidth]{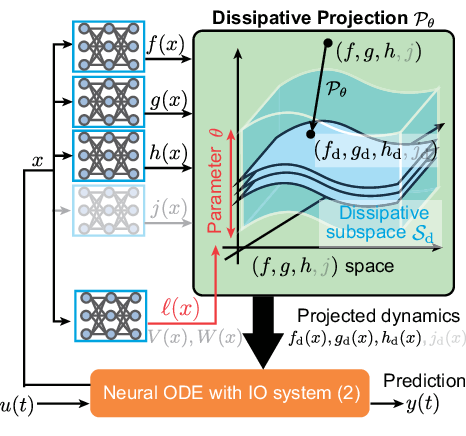}
    \caption{ Sketch of the proposed method: The dynamics of the input-output system $(\fn,\gn,\hn,\jn)$ is projected into a space with guaranteed dissipativity using dissipative projection $\mathcal{P}_\theta$, and the output signal $y(t)$ is predicted using the projected dynamics  $(\fm,\gm,\hmm,\jm)$ and the input signal $u(t)$.}
    \label{fig:parametric_projection}
\end{figure}

The aim of this study is to learn strictly dissipative dynamics fitted to input-output data. 
We consider the subspace of all dissipative dynamics within the function space consisting of tuples of four nonlinear maps $(\fs, \gs, \hs, \js)$ constructed by neural networks.
By projecting the dynamics $(\fs, \gs, \hs, \js)$ onto the subspace of dissipative dynamics, the resulting neural network-based dynamical system will inherently satisfy dissipativity. 
Consequently, by training the projected neural networks to fit the input-output data, both fitting accuracy and strict dissipativity are achieved.
Considering a parametric subspace of dissipative dynamics, we introduce the parameterized projection onto this subspace.

\begin{definition}[Dissipative projection]\label{Def:dissipative_projection}
Let $\mathbb{S}$ be a function space of $(\fs,\gs,\hs,\js)$, $\mathbb{S}_{\rm d} \subset \mathbb{S}$ be the subspace satisfying dissipativity, and $\Theta$ be a parameter set.
If the differentiable functional $\mathcal{P}_{\theta}: \mathbb{S}\rightarrow \mathbb{S}_{\theta} \subset \mathbb{S}_d$ satisfies
\begin{align}
\mathcal{P}_{\theta} \circ \mathcal{P}_{\theta}  = \mathcal{P}_{\theta },
\quad \mathbb{S}_{\rm d} = \bigcup_{\theta \in \Theta} \mathbb{S}_\theta,
\end{align}
then, $\mathcal{P}_{\theta}$ is called a dissipative projection.
\end{definition}

The nonlinear KYP lemma serves as a necessary and sufficient condition for dissipativity, meaning that $\mathbb{S}_{\rm d}$ aligns with the entirety of dynamics satisfying this condition.
By fixing the maps $(\ls,\Ws,V)$, we determine a subspace of dynamics $(f,g,h,j)$ that complies with the equation~(\ref{Eq:KYP_condition}) of the nonlinear KYP lemma.
By unfixing the maps $(\ls,\Ws,V)$, the union of all of such subspaces related to  $(\ls,\Ws,V)$ corresponds to the entire set of dissipative dynamics.
So, the maps $(\ls,\Ws,V)$ can be regarded as the parameter $\theta$ on Definition~\ref{Def:dissipative_projection}
(See Figure~\ref{fig:parametric_projection}).
 
By jointly learning the pre-projected dynamics $(\fs, \gs, \hs, \js)$ and the indirect parameter $\theta$, it becomes possible to optimize dissipative dynamics.
The formulation for learning strictly dissipative dynamics is established using the dissipative projection $\mathcal{P}_\theta$ as follows:

\begin{problem}
Let $D \triangleq \{\us_i, \ys^*_i\}_{i = 1}^{N}$ be a dataset and $\mathcal{P}_\theta$ be a dissipative projection.
Our problem is written as
\begin{align}
\begin{aligned}
    &\underset{(\fn,\gn,\hn,\jn)\in \mathbb{S},\theta \in \Theta}{\text{\bf minimize}}~\mathrm{E}_{(\us,\ys^*) \in D} [\|\ys^* - \ys\|^2] \\
\end{aligned} \label{Eq:optimization_projection_based}
\end{align}
where $\ys$ is the prediction result by the input signal $\us$ and the input-output system~(\ref{Eq:main_system}) from projected dynamics
\begin{align*}
(\fm,\gm,\hmm,\jm) \triangleq \mathcal{P}_\theta \big(\fn,\gn,\hn,\jn).
\end{align*}
\end{problem}

In the following sections, we analytically derive concrete dissipative projections $\mathcal{P}_\theta$ based on the nonlinear KYP lemma.
To derive the dissipative projection with parameters $(\ls,\Ws,V)$, we generally solve the equation~(\ref{Eq:KYP_condition}) in the nonlinear KYP lemma.
Therefore, in the next section, we derive the general solution to the matrix equation of the nonlinear KYP lemma, and in Section~\ref{Sec:Dissipative_Projection}, we derive a dissipative projection using the parameters \((\ls, \Ws, V)\) based on this general solution.
Finally, we introduce a loss function to realize efficient learning for dissipative input-output system~(\ref{Eq:main_system}).

\subsection{General Solution of Nonlinear KYP Lemma}\label{Sec:General_solution_of_Nonlinear_KYP_lemma}
For any maps $(\ls, \Ws, V)$, equations~(\ref{Eq:KYP_condition}) in the nonlinear KYP lemma are written as a quadratic matrix equation (QME) form of the dynamics~$(\fs,\gs,\hs,\js)$:
\begin{align}
    \Xs^\T \As \Xs + \Bs^\T \Xs + \Xs^\T \Bs  + \Cs = \zeros \label{Eq:QME}
\end{align}
where
\begin{align}
\begin{aligned}
    \Xs &\triangleq
\begin{bmatrix}
\fs & \gs\\
\hs & \js
\end{bmatrix},\hspace{1ex}
\As \triangleq
\begin{bmatrix}
\zeros & \zeros\\
\zeros & -\Qs
\end{bmatrix},\\
\Bs &\triangleq
\begin{bmatrix}
\frac{1}{2}\nabla V&  \zeros\\
\zeros &  -\Ss
\end{bmatrix},\hspace{1ex}
\Cs \triangleq  
\begin{bmatrix}
\|\ls\|^2 & \ls^\T\Ws\\
\Ws^\T \ls & \Ws^\T \Ws - \Rs
\end{bmatrix}.
\end{aligned}
\label{Eq:QME_parameters}
\end{align}
The general solution of this QME presents the following.
\begin{lemma}\label{Lem:solution_of_QME}
Assuming $\Qs$ is a negative definite matrix, if $\Rs- \Ss^\T \Qs^{-1}\Ss -\Ws^\T\Ws$ is a positive semi-definite matrix, then the QME~(\ref{Eq:QME}) exists a solution, and the general solution is written as
\begin{subequations}
\begin{align}
&\fs=  \PsC \fh + \frac{\nabla V}{\|\nabla V\|^2} \Big( \hh^\T \Qs \hh - \|\ls\|^2\Big)\label{Eq:QME_solution_f}\\
&\gs = \PsC \gh + 2\frac{\nabla V}{\|\nabla V\|^2} \Big(\hh^\T ( \Ss + \Qs \jh) - \ls^\T \Ws \Big)\label{Eq:QME_solution_g}\\
&\hs=\hh,\quad \js = \jh \label{Eq:QME_solution_hj}
\end{align}\label{Eq:QME_solution}
\end{subequations}
where $\PsC$ is the projection onto the complementary of  subspace spanned by the vector $\nabla V$ which defined as
\begin{align*}
     \PsC \triangleq \id_n  - \frac{1}{\|\nabla V\|^2} \nabla V \nabla V^\T.
\end{align*}
The intermediate variables $(\fh,\gh,\hh,\jh)$ are parameters in the solution of this QME~(\ref{Eq:QME}), such that $\jh$ satisfies the following ellipsoidal condition:
\begin{align}
\begin{aligned}
  &(\jh + \Qs^{-1}\Ss)^\T(-\Qs)(\jh + \Qs^{-1}\Ss)  \\
  &\hspace{10ex} =  \Rs - \Ss^\T \Qs^{-1}\Ss - \Ws^\T \Ws. 
\end{aligned}\label{Eq:ellipsoidal_equation}    
\end{align}
\end{lemma}
\begin{proof}
    See Appendix~\ref{APP:QME}.
\end{proof}

The solution of the QME is divided into the null space and non-null space of the matrix $A$ on the equation~(\ref{Eq:QME_parameters}).
Since $\left[\begin{smallmatrix}\fs&\gs\\0&0\end{smallmatrix}\right]$ corresponds to the null space of the matrix $\As$, it is the solution of the linear equation and $\PsC[\fh, \gh]$ is the complementary space.
Considering the non-null space of $A$, $\js$ is a matrix on the ellipsoidal sphere by a positive definite matrix $-\Qs$, and the existence condition is that the radius of this ellipsoidal $\Rs - \Ss^\T \Qs^{-1}\Ss - \Ws^\T \Ws$ is a positive semi-definite matrix.

This lemma assumes that $\Qs$ is a negative definite matrix, but it can similarly be shown when $\Qs$ is a positive definite matrix.
If $\Qs$ has eigenvalues with both positive and negative signs, or if a complementary space exists, the solution needs to be written for each eigenspace and becomes complicated.

The result of Lemma~\ref{Lem:solution_of_QME} implies that the entire set of dissipative dynamics $\mathbb{S}_{\rm d}$ is determined by the parameters $(\ls, \Ws, V)$ and the intermediate variable $(\fh, \gh, \hh, \jh)$ in the general solution.
Noting that $\Ws$ can be reduced as a map of $\jh$ derived from the ellipsoidal condition~(\ref{Eq:ellipsoidal_equation}), the entire set of dissipative dynamics $\mathbb{S}_{\rm d}$ is partitioned by only two parameters $(\ls, V)$.

In the next section, we explicitly derive the differentiable projection onto the parametric subspace $\mathbb{S}_{\ls, V}$ of dissipative dynamics, excluding the direct path from input to output.

\subsection{Projection onto Dissipative Dynamics Subspace}\label{Sec:Dissipative_Projection}
Based on Lemma~\ref{Lem:solution_of_QME}, this section derives the projection onto the subspace of dissipative dynamics $\mathbb{S}_{\ls, V}$ for any given mappings $\ls$ and $V$.
In many applications, the direct path from input to output $\js$ is often excluded ($\js\equiv 0$).
In such cases, the negative definite matrix $\Qs$ assumed in Lemma~\ref{Lem:solution_of_QME} is no longer required.
The following theorem proposes a projection of dissipative dynamics assuming $\js\equiv 0$.

If the direct path $\js$ is not excluded, it is necessary to construct a projection that satisfies the ellipsoidal condition~(\ref{Eq:ellipsoidal_equation}) for $\js$.
The general case of $\js$, the projection onto the subspace of dissipative dynamics $\mathbb{S}_{\ls, V}$ is shown in Appendix~\ref{APP:general_dissipative_projection}.

\begin{theorem}[Dissipative Projection]~\label{Thm:Dissipative_Projection}
Assume that $R$ is positive semi-definite matrices. 
The following map $\mathcal{P}_{\ls,V} : (\fn,\gn,\hn ) \mapsto (\fm,\gm,\hmm)$ satisfying
\begin{align}
\fm &= \PsC \fn + \frac{\hn^\T \Qs \hn - \| \ls\|^2 }{\|\nabla V\|^2} \nabla V,\\
\gm &= \PsC \gn + \frac{2}{\|\nabla V\|^2} \nabla V \big( \hn^\T \Ss - \ls^\T \sqrt{\Rs}\big),\\
\hmm &= \hn
\end{align}
is a dissipative projection, where $\ls : \R^n \rightarrow\R^{m}$ and $V: \R^n \rightarrow \R_{\geq 0}$.
\end{theorem}
\begin{proof}
See Appendix~\ref{APP:Proof_of_main_theorem}.
\end{proof}
Note that $\sqrt{R}$ is the root of a positive semi-definite matrix $R$, satisfying $\sqrt{R} \sqrt{R} = R$ and $\sqrt{R}$ is symmetric.

Projections that strictly guarantee internal stability, input-output stability, and energy conservation are achieved by designing the supply rate parameters $(\Qs, \Ss, \Rs)$.
The projection that guarantees internal stability coincide with the literature~\cite{Manek2019}, and the projection that guarantees input-output stability corresponds with another study~\cite{kojima2022learning}.
For details on the differences from previous studies, please refer to the Appendix~\ref{App:Difference_of_our_Previous Study}.
Below, we demonstrate differentiable projections that guarantee each of these time-series characteristics.

\begin{corollary}[Stable Projection] \label{Cor:Proj_Stable}
The following map $\mathcal{P}_{V} : (\fn,\gn,\hn ) \mapsto (\fm,\gm,\hmm)$:
\begin{align}
\begin{aligned}
\fm &= \fn  -  \frac{\nabla V}{\|\nabla V\|^2} \mathrm{ReLU}(\nabla V^\T \fn) ,\\
\gm&= \gn,\quad \hmm = \hn    
\end{aligned}
\end{align}
is a projection into stable dynamics.
\end{corollary}
\begin{proof}
When $Q=R=S=0$, it is derived from the theorem.
\end{proof}
\begin{corollary}[Input-Output Stable Projection] \label{Cor:Proj_L2Stalbe}
The following map $\mathcal{P}_{\ls,V} : (\fn,\gn,\hn ) \mapsto (\fm,\gm,\hmm)$:
\begin{align}
\begin{aligned}
\fm &= \PsC \fn -\frac{\nabla V}{\|\nabla V\|^2}\big(\| \hn \|^2 + \| \ls\|^2\big),\\
\gm &= \PsC \gn -2\gamma \frac{\nabla V\ls^\T}{\|\nabla V\|^2}\\
\hmm &= \hn
\end{aligned}
\end{align}
is a projection into input-output ($\mathscr{L}_2$) stable dynamics and the $\gamma >0$ is the input-output gain.
\end{corollary}
\begin{proof}
When $Q = -I_l$, $S=0$, and $R=\gamma^2 I_m$, it is derived from the theorem.
\end{proof}

The definition of dissipativity is expressed as an inequality involving the integral of the supply rate and the change of the storage function. 
Assuming $\Rs=0$ and $\ls\equiv 0$, this becomes an equality condition.
This allows for the construction of a projection that strictly preserves the energy conservation law.

\begin{corollary}[Energy Conservation Projection]~\label{Thm:Energy_Conservation_Projection}
Assuming $R=0$, if the following mapping $\mathcal{P}_{V} : (\fn,\gn,\hn ) \mapsto (\fm,\gm,\hmm)$ is given by
\begin{align}
\begin{aligned}
\fm&= \PsC \fn + \frac{\nabla V}{\|\nabla V\|^2}\hn^\T \Qs \hn,\\
\gm &= \PsC \gn + 2\frac{\nabla V}{\|\nabla V\|^2}\hn^\T \Ss,\\
\hmm &= \hn
\end{aligned}
\label{Eq:energy_conservation_projection}
\end{align}
then the input-output system~(\ref{Eq:main_system}) satisfies 
\begin{align*}
V(\xs(t_1)) - V(\xs(t_0)) = \int_{t_0}^{t_1}w\big(\us(s),\ys(s)\big)\mathrm{d}s.    
\end{align*}
\end{corollary}
\begin{proof}
See Appendix~\ref{APP:Proof_of_Energy_Conservation_Projection}.
\end{proof}

The energy conservation projection supports the Hamiltonian equations, which conserve energy, and the port-Hamiltonian systems, where energy exchange is explicitly defined.
In this context, the storage function $V$ corresponds to the Hamiltonian, and the supply rate $w(u,y)$ corresponds to the energy dissipation in the port-Hamiltonian system.

A similar concept to dissipativity in evaluating input-output systems is passivity.
Since passivity can be described as the exchange of energy, it can naturally be addressed in this study by adjusting the supply rate parameter for dissipativity (see Appendix~\ref{APP:Passivity})

Note that dissipative projections are not unique because they depend on space metrics.
Additionally, since the dissipative constraint is nonlinear, explicitly describing the underlying space metric is difficult.
For instance, existing study~\cite{kojima2022learning} presents projections onto a non-Hilbert metric spaces under a simple quadratic constraint called input-output stability.
For further details, please refer to the Appendix~\ref{APP:Another_Dissipative_Projection}.

\begin{table*}[t]
    \centering
    \begin{tabular}{|c|c||r|r|r|r|r|}
    \hline 
     Train & Test & Naive  & Stable&  IO stable& Conservation & Dissipative \\ \hline 
     \hline 
     \multirow{3}{*}{\begin{tabular}{c}Rectangle\\(N=100)\end{tabular}}&Rectangle  
                       &  $0.250 \pm 0.184$ & $0.252 \pm 0.181$ & $0.194 \pm 0.095$ & ${\bf 0.077} \pm 0.066$ & $0.212 \pm 0.144$ \\ \cline{2-7}
     &Step       &  $0.205 \pm 0.195$ & $0.240 \pm 0.203$ & $0.225 \pm 0.115$ & ${\bf 0.046} \pm 0.021$ & $0.197 \pm 0.147$ \\\cline{2-7}
     &Random       &  $0.049 \pm 0.044$ & $0.047 \pm 0.037$ & $0.068 \pm 0.036$ &  ${\bf0.023} \pm 0.031$ & $0.040 \pm 0.028$ \\\hline
     \multirow{3}{*}{\begin{tabular}{c}Rectangle\\  (N=1000)\end{tabular}}&Rectangle 
                       &  ${\bf 0.029} \pm 0.000$ & $0.029 \pm 0.000$ & $0.029 \pm 0.000$ & $0.029 \pm 0.000$ & $0.060 \pm 0.061$ \\ \cline{2-7}
     & Step      &  $0.024 \pm 0.000$ & $0.024 \pm 0.001$ & $0.024 \pm 0.001$ & ${\bf 0.024} \pm 0.001$ & $0.039 \pm 0.029$ \\\cline{2-7}
     & Random      &  ${\bf 0.005} \pm 0.001$ & $0.009 \pm 0.003$ & $0.007 \pm 0.002$ & $0.006 \pm 0.003$ & $0.021 \pm 0.030$ \\\hline     
     \end{tabular}
    \caption{The prediction error (RMSE) of the mass-spring-damper benchmark}
    \label{tab:result_linear}
\end{table*}

\subsection{Loss function}
The optimization problem~(\ref{Eq:optimization_projection_based}) based on the dissipative projection requires careful attention to learning methods, as there is a degree of freedom in the internal parameters.
Here, we define the regularized loss function as follows:
\begin{align}
     \mathrm{Loss} \triangleq \mathrm{E}_\mathcal{D}\Big[ \|y^* - y\|^2 \Big] + \lambda_1 \mathcal{L}_{\rm proj} + \lambda_2 \mathcal{L}_{\rm recons}, \label{Eq:Loss}
\end{align}
where $\lambda_1$ and $\lambda_2$ are positive hyperparameters.
In the first term, the squared error of the data point $y^*$ and the prediction result $y$ can be computed by solving the neural ODE represented as the equation~(\ref{Eq:main_system}).

The second term $\mathcal{L}_{\rm proj}$ prevents the distance before and after projection from diverging by reducing a degree of freedom of projection, that is,
\begin{align*}
     \mathcal{L}_{\rm proj} &\triangleq  \mathrm{E}\Big[\left\|
     (\mathrm{id}_\mathbb{S}-\mathcal{P}_\theta)(f,g,h,j)
     \right\|^2 \Big]\\
     &= \mathrm{E}_{x\sim \mathcal{N}}\Big[ \|\fn(x) - \fm(x) \|^2 + \|\gn(x) - \gm(x) \|^2 \Big],
\end{align*}
where $\mathrm{id}_\mathbb{S}$ is the identity map on the set of dynamics $\mathbb{S}$.
In our implementation of Theorem~\ref{Thm:Dissipative_Projection}, only $f$ and $g$ are involved in the projection.
Hence, $\mathcal{L}_{\rm proj}$ can be calculated by the last formula, where the expectation by sampling from an $n$-dimensional standard normal distribution $\mathcal{N}$.
In the following experiments, we use 100 samples to compute this expected values.
Here, we emphasize that this loss term is used merely to reduce the degrees of freedom, even if its value is non-zero, our projected dynamics  are always dissipative.

The last term is to prevents $h$ from becoming degenerate in the early stages of gradient-based learning, that is,
\begin{align*}
      \mathcal{L}_{\rm recons} &\triangleq \mathrm{E}_{x\in\mathcal{X}}\Big[ \|x - \eta(h(x))\|\Big],
\end{align*}
where $\eta: \R^l \rightarrow \R^n$ is represented by an additional neural network and $\mathcal{X}$ is a set of $\eta(x)$ corresponding to $y$ in the the first term.
where the reconstruction map $\eta: \R^l \rightarrow \R^n$ is represented by an additional neural network and $\mathcal{X}$ is a set of samples on the solution of the neural ODE when computing the first term. 
The reconstruction term is intended to prevent $h$ from becoming a trivial function $h(x)\equiv 0$ during learning. In our gradient-based learning, we initialize the neural network weights to values close to $0$, and start learning from internal-state behaviors $x(t)$ around $0$.
In that case, $h$ is learned as $h(x)\equiv 0$ at an early stage and does not change. 
The reconstruction term is intended to prevent this, empirically.

In this study, the map $\ls$ in the dissipative projection, the reconstruction map $\eta$, and nonlinear dynamics $(f, g, h)$ are parameterized by using neural networks.
All of the neural networks are trained using this loss function (\ref{Eq:Loss}).

\section{Result} \label{Sec:Result}
\begin{figure}[t]
    \centering
     \includegraphics[width=1.0\linewidth]{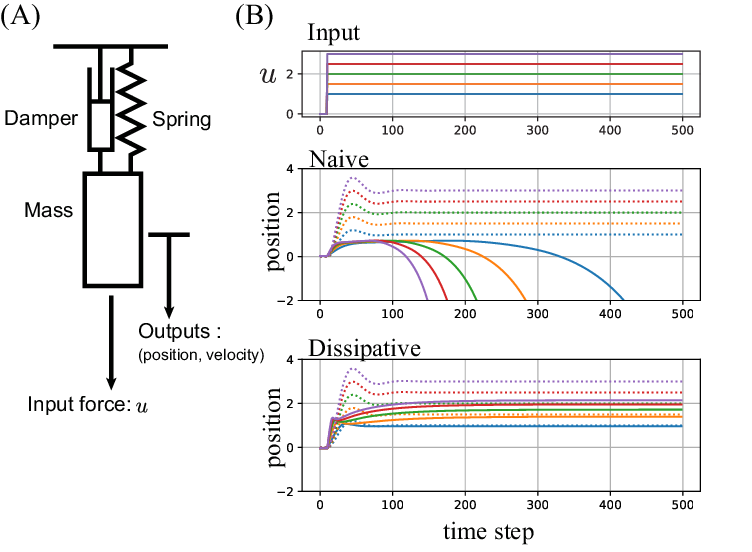}
    \caption{(A) Sketch of mass-spring-damper system.
    (B) Prediction results for out-of-domain inputs, i.e., five long step signals with different amplitudes.
    The top figure is the input signal behaviors, the middle figure shows the position of the mass predicted by the model trained using the naive model, and the bottom figure shows the position predicted by the proposed dissipative model.
    The dashed lines are the plots of the ground truth, and each color of lines shows the results with the same input signals.
    }
    \label{fig:linear_result_long}
\end{figure}

We conducted three experiments to evaluate our proposed method.
The first experiment uses a benchmark dataset generated from a mass-spring-damper system, which is a classic example from physics and engineering.
In the next experiment, we evaluate our methods by an $n$-link pendulum system, a nonlinear dynamical system related to robotic arm applications.
Finally, we applied our method to learning an input-output fluid system using a fluid simulator.

We evaluate the prediction error at each time point using root mean square error, which we call RMSE$(t)$, in the output domain.
The aim is to observe the error affected by satisfying the dissipative system over time.
In addition, we used the time-averaged RMSE$(t)$, which we refer to simply as RMSE, to evaluate the prediction error of models.
In the following experiments, we performed experiments on the trained model by changing the input, such as test input signal separated from training data, signals with different input lengths, and signals of different types, using these evaluation metrics.
In the following experiment section, $N$ denotes the number of pairs of input-output signals for training.

We retry five times for all experiments and show the mean and standard deviations of the both metrics.
For simplicity in our experiments, the sampling step $\Delta t$ for the output $y$ is set as constant and the Euler method is used to solve neural ODEs.
The initial state $x_0$ in this ODE is fixed as $0$ for simplicity.
The hyperparameters including the number of layers in the neural networks, the learning rate, optimizer, and the weighted decay are determined using the tree-structured Parzen estimator (TPE) implemented in Optuna  \cite{optuna_2019} (see Appendix~\ref{APP:result_bo}).

\subsection{Mass-spring-damper Benchmark} \label{Sec:linear_result}

\begin{figure}[t]
    \centering
     \includegraphics[width=1.0\linewidth]{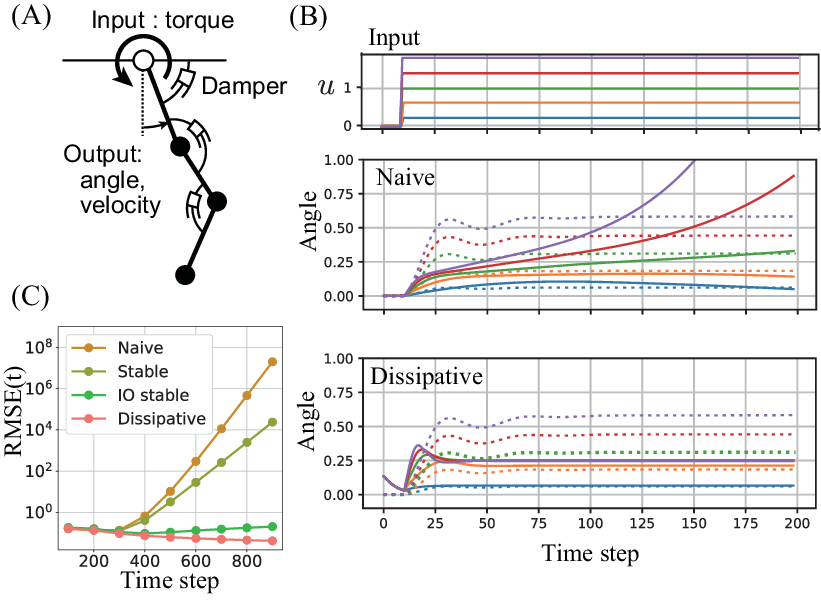}
    \caption{(A) Sketch of $n$-link pendulum model.
    (B) Prediction results for out-of-domain inputs.
    The top figure is the input signal behaviors, the middle figure shows the angle of the first pendulum predicted by the model trained using the naive model, and the bottom figure shows the angle predicted by the proposed dissipative model.
    The dashed lines are the plots of the ground truth, and each color of lines shows the results with the same input signals.
    (C) RMSE$(t)$ related to long input step signal.
    }
    \label{fig:nlink_result_long}
\end{figure}

\begin{figure*}[t!]
    \centering
     \includegraphics[width=0.9\linewidth]{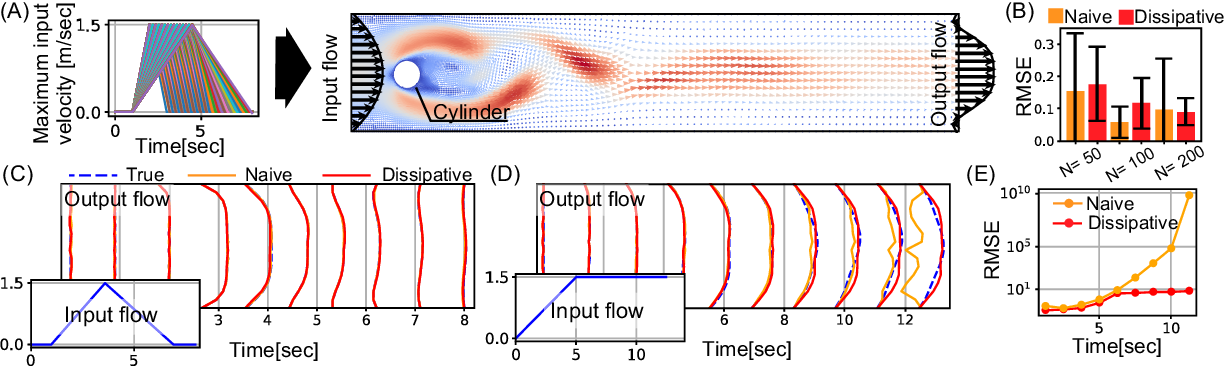}
    \caption{
(A) Sketch of the flow around cylinder model.
(B) Time-average RMSE of test triangular (in-domain) waves by changing the number of training inputs$N$.
(C) The predicted output flow for a test triangular (in-domain) wave with $N=100$. 
Each curved line represents the spatial distribution of the output flow at each time point.
The blue, orange, and red lines shows the ground truth, the output flow predicted by the naive model, and the output predicted by our dissipative model, respectively.
(D) The predicted output flow for a test clip (out-of-domain) wave with $N=100$.
(E)  RMSE$(t)$ for long time simulation with test clip waves.
    }
    \label{fig:flow_around_cylinder}
\end{figure*}

We generate input and output signals for this experiment by a mass-spring-damper system.
This dynamics was chosen as the first simple example because it is linear and its properties can be easily understood analytically.

Table~\ref{tab:result_linear} shows predictive performance comparing a naive model, the proposed stable model (Corollary~\ref{Cor:Proj_Stable}),
input-output stable model (Corollary~\ref{Cor:Proj_L2Stalbe}), energy conservation model (Theorem~\ref{Thm:Energy_Conservation_Projection}), and dissipative model (Theorem~\ref{Thm:Dissipative_Projection}).
The naive model simply use neural networks as $(f,g,h)$, which is trained by minimizing the squared error.
The first and fourth rows of this table shows the results of evaluation using $N$-inputs rectangle signals for training data and different $0.1 \times N$ rectangle input signals for testing.

Since our focus is on input-output systems, the model may be influenced by the type of input.
Therefore, we considered a scenario where only data from input rectangle waves could be collected during training and evaluated the model by input signals generated by step functions and random walks, in addition to rectangle waves, during testing.
We call such scenarios out-of-domain cases, shown in Table~\ref{tab:result_linear}.
The hyperparameters related to these models and the detail of the experimental setting are described in the Appendix~\ref{APP:exaple_linear}.

The proposed conservation and dissipative models exhibited high predictive performance with unforeseen inputs at $N=100$.
In particular, conservation showed a statistically significant improvement over the naive model.
This is because the conservation model utilizes the energy relationships in the most rigorous manner.
Note that the supply rate $w$ includes not only the terms corresponding to increases in internal energy due to external forces but also the terms corresponding to energy dissipation by dampers.
When the data size was increased to 1000, predictive accuracy differences between the methods were small.
These observations suggest that enforcing dissipative or conservation properties ensures high predictive performance, particularly with smaller data sizes, due to better match with the underlying data-generating system.

Figure~\ref{fig:linear_result_long} illustrates another out-of-domain case when unexpectedly longer step inputs (1000 steps), exceeding those used during training (100 steps), are given.
The results show that while the naive method may diverge with such an unexpected long input, ensuring dissipative properties allows the output to remain bounded.

\subsection{$n$-link Pendulum Benchmark} \label{Sec:nlink_result}

Next, to demonstrate the nonlinear case,
we adopt the $n$-link pendulum system, characterized by multiple connected pendulums.
The movement of each link is governed by nonlinear equations of motion, leading to extremely complex overall system behavior (Figure~\ref{fig:nlink_result_long}~(A)).

Figure~\ref{fig:nlink_result_long}~(B) and (C) show the behavior of the system when an input different from the 100-step rectangle waves used during training is input.
The results show that divergence is evident from around 200 steps, indicating that naive and stable models may diverge when receiving an unexpectedly long input, but that the output is appropriately bounded when IO stability and dissipative are guaranteed.
All numerical experiment results, including errors for the same type of input as used during training, are listed in the Appendix~\ref{APP:exaple_nlink}.

\subsection{Fluid System Benchmark}
\label{Sec:flow_result}

In the final part of this study, we aim to predict the input-output relationship of fluid flow around a cylinder (Figure~\ref{fig:flow_around_cylinder} (A))\cite{Schäfer1996}.
This phenomenon involves complex behaviors, such as periodic oscillations and flow instabilities, resulting from the formation of K\'{a}rm\'{a}n vortex streets.
In this experiment, the left and right flow velocities are spatially discretized into 16 divisions, which are the inputs and outputs of this system.

We constructed a predictive model that guarantees dissipativity based on the results of fluid simulations using a triangular wave input (detailed conditions are provided in the Appendix~\ref{APP:flow_around_cylinder}).
The prediction results with test triangular wave inputs showed good accuracy for \( N = \{50, 100, 200\} \), and a naive neural network also demonstrated comparable accuracy (Figure.~\ref{fig:flow_around_cylinder} (B),(C)).
Using out-of-domain clipped wave inputs, we compared the trained predictive model against long-term simulation.
The results showed that the model maintained good predictive accuracy even for extended prediction periods, outperforming the naive model. 
The RMSE$(t)$ begins to increase at the time matching the training signal length (8 seconds) and becomes more pronounced over longer input signal (Figure.~\ref{fig:flow_around_cylinder} (D),~(E)).

\section{Conclusion}
In this study we analytically derived a general solution to the nonlinear KYP lemma and a dissipative projection.
Furthermore, we showed that our proposed methods that strictly guarantee dissipativity including internal stability, input-output stability, and energy conservation.
Finally, we confirmed the effectiveness of our method, particularly its robustness to out-of-domain input, using both linear systems and complex nonlinear systems, including an $n$-link pendulum and fluid simulations.
A limitation of this study is the requirement to determine the dissipative hyperparameters based on the rough properties of the target system.
For future work, this research will lead to system identification to control real-world dissipative systems.

\section*{Acknowledgments}
This research was supported by JST Moonshot R\&D Grant Number JPMJMS2021 and JPMJMS2024.
This work was also supported by JSPS KAKENHI Grant No.21H04905 and CREST Grant Number JPMJCR22D3, Japan.

\bibliography{references}

\clearpage
\appendix

\section{Proof of the nonlinear KYP lemma {\cite[Lemma 4.101]{brogliato2020dissipative}}}\label{APP:Proof_of_KYP_lemma}

{\bf Sufficient Condition}~:~
Using nonlinear KYP condition~(\ref{Eq:KYP_condition}), the time derivative of $V$ along the dynamics~(\ref{Eq:main_system}) is calculated as
\begin{align}
\dot{V}(x) &=  w(u,y) + \|\ls(x) + \Ws(x) u\|^2.\label{Eq:non_linear_KYP_lemma_V_dot}
\end{align}
Therefore, the integral becomes a dissipative condition~(\ref{Eq:dissipativity}).
Note that, if $l(x) \equiv 0$ and  $W(x)\equiv 0$, the dynamics~(\ref{Eq:main_system}) satisfies equivalent condition of dissipativity~(\ref{Eq:dissipativity}) :
\begin{align*}
V(\xs(t_1)) - V(\xs(t_0)) =\int_{t_0}^{t_1}w\big(\us(s),\ys(s)\big)\mathrm{d}s.
\end{align*}

{\bf Necessary Condition}~:~
If the system is dissipative, there exist $V :\R^n \rightarrow  \Rnn$ such that the following equation is satisfies:
\begin{align}
\begin{aligned}
d(u,y) &\triangleq w(u,y) - \dot{V}(x) \\
&= w(u,y) - \nabla V (f(x)+g(x)u)\\
&\geq 0.
\end{aligned}\label{Eq:proof_nonlinar_KYP_def_d}
\end{align}
Since the supply rate $w$ is at least quadratic in $u$, $d(u,y)$ is a quadratic and non-negative function of $u$.
Hence, there exists $\ls:\R^n\rightarrow \R^q$ and $W:\R^n\rightarrow \R^{q\times m}$ such that
\begin{align}
    d(u,y) = \big( W(x)u+ l(x)\big)^\T\big( W(x)u+ \ls(x)\big).\label{Eq:proof_nonlinar_KYP_d_quadratic}
\end{align}
The two equations (\ref{Eq:proof_nonlinar_KYP_def_d}), (\ref{Eq:proof_nonlinar_KYP_d_quadratic}) are organized by order of $u$, the condition of the non-linear KYP lemma~(\ref{Eq:KYP_condition}) is derived.
\hspace{\fill}$\Box$
\section{Freedom Degree of Dissipative Systems and Nonlinear KYP Lemma Condition}\label{APP:Freedom_degree_of_KYP_lemma}
In the state-space model, the storage function $V$ representing dissipativity possesses degrees of freedom.
These degrees of freedom arise from two sources: (i) the definition of dissipativity and (ii) the identity of the state-space model.
Regarding the definition of dissipativity~(\ref{Eq:dissipativity}), there is a degree of freedom concerning the bias of $V$, as the focus is only on the difference in the storage function.
Additionally, in state-space models, the same input-output relationship is maintained under any diffeomorphism map of the state $x$, leading to a degree of freedom in the storage function $V$ within the input-output system.
Consequently, as long as the storage function $V$ is diffeomorphic to the system's dynamics, it can be freely defined, and its bias can be freely designed.

Furthermore, there is a degree of freedom in the choice of $\ls$ and $\Ws$ used in the nonlinear KYP lemma.  
In addition to the degree of freedom associated with the diffeomorphism, there exists flexibility in selecting the dimension $q$ of $\ls$ and $\Ws$.  
In general, determining the optimal $q$ that ensures the uniqueness of $\ls$ and $\Ws$ remains an open problem. 
However, in this study, setting $q = m$ is sufficient to ensure adequate expressiveness.

\section{Special case of nonlinear KYP Lemma}\label{APP:Special_case_KYP_lemma}
Before the advent of the nonlinear KYP lemma, criteria for passivity and input-output stability had already been proposed.
For passivity, the Lur’e equation was used (See Appendix~\ref{APP:Passivity}), while input-output stability was determined by the Hamilton-Jacobi inequality.
These criteria are now considered special cases of the nonlinear KYP lemma.

\begin{corollary}[Hamilton-Jacobi inequality]
If there exists a differentiable function $V:\R^n \rightarrow \R$ satisfies the following condition:
\begin{align}
\nabla V^\T(x) f(x)   - \frac{1}{4\gamma^2}\|\nabla V^\T(x) g(x)\|^2-  \|h(x)\|^2 \leq 0,
\end{align}
then consider the input-output system~(\ref{Eq:main_system}) without direct-path is $\mathscr{L}_2$ stable, i.e. , there exists $\beta:\R^n\rightarrow \Rnn$ such that
\begin{align*}
    \|y\|_{\mathscr{L}_2} \leq \gamma \|u\|_{\mathscr{L}_2}  +  \beta(x_0).
\end{align*}
where $\|  x \|_{\mathscr{L}_2}\triangleq  \sqrt{\int_0^\infty x^\T(t)x(t) \dd t}$.
\end{corollary}
Note that the reason why the equality may not always hold in the Hamilton-Jacobi inequality is the degree of freedom of $q$, which is dimension of $l(x)$ and $W(x)$.
\section{Proof of General Solution about QME (\ref{Eq:QME})}\label{APP:QME}
Assuming $\Qs$ is a negative definite matrix, we split the null space of $\As$ and the other, given by
\begin{align*}
    \Ps &\triangleq
    \begin{bmatrix}
    \Ps_+\\
    \Ps_-
    \end{bmatrix}
     =
    \begin{bmatrix}
    \id_{n}& \zeros \\
    \zeros& \sqrt{-\Qs}
    \end{bmatrix}\\
    \begin{bmatrix}
    \Ys_+\\
    \Ys_-
    \end{bmatrix}
    &\triangleq
    \Ps \Xs    = 
    \begin{bmatrix}
    \fs&\gs\\
    \sqrt{-\Qs}\hs&\sqrt{-\Qs}\js
    \end{bmatrix}, \\
    \begin{bmatrix}
    \Bs_+\\
    \Bs_-
    \end{bmatrix}
    &\triangleq
    \Ps^{-1} \Bs 
    =
    \begin{bmatrix}
    \frac{1}{2}\nabla V& \zeros\\
    \zeros &  - (-\Qs)^{-\frac{1}{2}} \Ss
    \end{bmatrix}.
\end{align*}
Since $\As = \Ps_-^\T \Ps_-$, the QME is rewritten as
\begin{align*}
&\Xs^\T \As \Xs + \Bs^\T \Xs + \Xs^\T \Bs + \Cs\\
&= \Xs^\T \Ps_-^\T \Ps_- \Xs + \Bs^\T \Ps^{-1} \Ps\Xs + \Xs^\T \Ps  \Ps^{-1} \Bs +  \Cs \\
&= \Ys_-^\T \Ys_- +  \Bs^\T_+ \Ys_+ +  \Ys^\T_+ \Bs_+ +  \Bs^\T_- \Ys_- +  \Ys^\T_- \Bs_- + \Cs\\
&= \zeros.
\end{align*}
Splitting to two equations depending on $\Ys_+$ and $\Ys_-$ , the previous equation is equivalent to the following three equations:
\begin{subequations}
    \begin{align}
         &\Bs^\T_+ \Ys_+ +  \Ys^\T_+ \Bs_+ +  \Cs_+ = 0 \label{Eq:Null_QME}\\
         &Y_-^\T Y_-  + B^\T_- Y_- +  Y^\T_- B_- - C_- = 0 \label{Eq:Non_Null_QME}\\
         &C_+ - C_- = C  \label{Eq:bias_QME}
    \end{align}\label{Eq:Split_QME}
\end{subequations}
where  $C_+$ and $C_-$ are indirect variables that belong in the set of biases for which a solution exists in Eq.~(\ref{Eq:Null_QME}) and (\ref{Eq:Non_Null_QME}), respectively.
Here, we denote symmetric matrices $C_+$ and $C_-$ as 
\begin{align*}
C_+ \triangleq
\begin{bmatrix}
c_+^{11}&c_+^{12}  \\
(c_+^{12})^\T&c_+^{22}
\end{bmatrix},\quad
C_- \triangleq
\begin{bmatrix}
c_-^{11}&c_-^{12}  \\
(c_-^{12})^\T&c_-^{22}
\end{bmatrix}
\end{align*}
where $c_*^{11} $ are scalers, $c_*^{12}$ are $m$ dimensional row vectors, and $c_*^{22}$ are $m$ dimensional symmetric matrices.
Equations~(\ref{Eq:Split_QME}) are solved the followings, respectively:

{\bf Equation~(\ref{Eq:Null_QME})}~:~
The equation on the null space of $A$ for each element is written as
\begin{align*}
&\Bs^\T_+ \Ys_+ +  \Ys^\T_+ \Bs_+ +  \Cs_+ \\
&= 
\begin{bmatrix}
\nabla V^\T f &  \frac{1}{2}\nabla V^\T g \\
\frac{1}{2} g^\T \nabla V & 0 
\end{bmatrix}
+
\begin{bmatrix}
c_+^{11}&c_+^{12}  \\
(c_+^{12})^\T&c_+^{22}
\end{bmatrix} \\
&= 0.
\end{align*}
This becomes a linear equation in $[f,~g]$, so the general solution is given by:
\begin{align}
[f,~g] &=  (I_n - P_{\nabla V}) [\fh~,\gh] - \frac{1}{\|\nabla V\|^2} \nabla V [c_+^{11},2c_+^{12}]\label{Eq:QME_fg}\\
c^{22}_+ &= 0 \label{Eq:QME_Cplus}
\end{align}
where  $P_{\nabla V}$ is the projection onto the $\{ \alpha  \nabla V : \alpha \in \R \} $ subspace, defined as $P_{\nabla V} \triangleq \frac{1}{\|\nabla V\|^2}\nabla V \nabla V^\T$, and $\fh\in \R^n$ and $\gh\in \R^{n\times m}$ are degree of freedom on this equation~(\ref{Eq:Null_QME}).

{\bf Equation~(\ref{Eq:Non_Null_QME})}~:~
The equation on the non-null space of $A$ for each element is written as quadratic form :
\begin{align*}
&(Y_-+ B_-)^\T (Y_-+ B_-) 
= C_- + B_-^\T B_-\\
&=
\begin{bmatrix}
c_-^{11} & c_-^{12}\\
(c_-^{12})^\T &  c_-^{22} -  S^\T Q^{-1}S 
\end{bmatrix}
\end{align*}
If the previous solution is defined, there exists 
\begin{align}
Z = U [z_1,z_2]\in \R^{l\times (m+1)}    
\end{align}
such that the following matrix decomposition exists : 
\begin{align*}
Z^\T Z = 
\begin{bmatrix}
z_1^\T z_1 & z_1^\T z_2\\
z_2^\T z_1 & z_2^\T z_2
\end{bmatrix}
 = 
\begin{bmatrix}
c_-^{11} & c_-^{12}\\
(c_-^{12})^\T &  c_-^{22} -  S^\T Q^{-1}S, 
\end{bmatrix}
\end{align*}
where $z_1$ is a $l$ dimensional vector, $z_2$ is a $l\times m$ dimensional matrix, and $U\in\R^{ l\times l}$ is an orthogonal matrix, which is a degree of freedom on this quadratic equation~(\ref{Eq:Non_Null_QME}).
The solution of $Y_-$ is derived as
\begin{align*}
    Y_- &=  Z - B_- \\
    &= U[z_1~ z_2]  -  [0,~ - (-Q)^{-\frac{1}{2}}S]
\end{align*}
Hence, the solution of $X$ is given by
\begin{align}
    [h,~j] &= (-Q)^{-\frac{1}{2}} U[z_1,~z_2] - [0,Q^{-1}S],\label{Eq:QME_hj}
\end{align}
and the $C_-$ satisfies
\begin{align}
C_- = 
\begin{bmatrix}
z_1^\T\\
z_2^\T
\end{bmatrix}
[z_1,~z_2]
 +
\begin{bmatrix}
0 & 0\\
0 & S^\T Q^{-1}S
\end{bmatrix}    \label{Eq:QME_Cminus}
\end{align}

{\bf Equation~(\ref{Eq:bias_QME})}~:~ 
The biases $C_+, C_-$ satisfies the equation~(\ref{Eq:QME_Cplus}) and (\ref{Eq:QME_Cminus}),respectively, the bias equation~(\ref{Eq:bias_QME}) for each element is written as
\begin{align*}
&C_+ - C_- = C \\
&\Leftrightarrow
\begin{matrix}
\begin{bmatrix}
c_+^{11}&c_+^{12}  \\
(c_+^{12})^\T&0
\end{bmatrix}
-
\left(
\begin{bmatrix}
z_1^\T\\
z_2^\T
\end{bmatrix}
[z_1,~z_2]
 +
\begin{bmatrix}
0 & 0\\
0 & S^\T Q^{-1}S
\end{bmatrix}
\right)\\
\hspace{25ex}= 
\begin{bmatrix}
\ls^\T \\
\Ws^\T 
\end{bmatrix}
[\ls~\Ws] - 
\begin{bmatrix}
0 & 0\\
0 & \Rs
\end{bmatrix}
\end{matrix}\\
&\Leftrightarrow
\begin{bmatrix}
    \ls^\T\\
    \Ws^\T
\end{bmatrix}
[\ls~\Ws]
+
\begin{bmatrix}
    z_1^\T\\
    z_2^\T
\end{bmatrix}
[z_1~z_2]
= 
\begin{bmatrix}
    c_1&c_2\\
    c_2^\T& \Rs - \Ss^\T \Qs^{-1}\Ss
\end{bmatrix}.
\end{align*}

{\bf Summarize equations:~}~
To summarize the general solution of the divided QMEs (\ref{Eq:Split_QME}), the following holds true for each element.
\begin{align*}
f &=  (I_n - P_{\nabla V}) \fh - \frac{1}{\|\nabla V\|^2} \nabla V \big(\|\ls\|^2 + \|z_1\|^2\big)\\
g &=  (I_n - P_{\nabla V}) \gh - \frac{2}{\|\nabla V\|^2} \nabla V \Big(\ls^\T W + z_1^\T z_2\Big) \\
[h,j] &= (-Q)^{-\frac{1}{2}} U[z_1,~z_2] - [0,~Q^{-1}S],\\
z_2^\T z_2 &= R - S^\T Q^{-1}S - W^\T W.
\end{align*}
If we convert the degree of freedom from $(z_1, z_2)$ to $(\hh, \jh)$, written as 
\begin{align*}
    \hh =  (-Q)^{-\frac{1}{2}} Uz_1,\quad  \jh =  (-Q)^{-\frac{1}{2}} U z_2  - \Qs^{-1} \Ss
\end{align*}
then the general solution is written as 
\begin{align*}
&\fs= (\id_n - \Ps_{\nabla V}) \fh+ \frac{\nabla V}{\|\nabla V\|^2} \big( \hh^\T \Qs \hh - \|\ls\|^2\big)\\
&\gs= (\id_n - \Ps_{\nabla V}) \gh + 2\frac{\nabla V}{\|\nabla V\|^2} \big(\hh^\T ( \Ss + \Qs \jh) - \ls^\T \Ws \big)\\
&[\hs, \js]=[\hh, \jh],\\
&(\jh + \Qs^{-1}\Ss)^\T(-\Qs)(\jh + \Qs^{-1}\Ss)\\
&\hspace{20ex}=  \Rs - \Ss^\T \Qs^{-1}\Ss - \Ws^\T \Ws.
\end{align*}
As $\Rs - \Ss^\T \Qs^{-1}\Ss - \Ws^\T \Ws$ corresponds to the radius of the ellipsoid, and $\jh$ exists if it is a positive semi-definite matrix.
It is the condition for the existence of this QME solution.
Note that the notation $\PsC \triangleq \id_n  -  \Ps_{\nabla V}$ is used in the main to keep the description short. \hspace{\fill}$\Box$

\section{General Dissipative Projection}\label{APP:general_dissipative_projection}
As a preliminary step, we provide the projection of a matrix onto an ellipsoid.
Considering a positive definite matrix $A\in \R^{n\times n}$, $B\in \R^{n\times m}$, and a positive semi-definite matrix $C\in \R^{m\times m}$, the ellipsoid $\mathbb{E}$ is defined as
\begin{align}
 \mathbb{E} \triangleq  \{ X\in \R^{n\times m} :  (X - B)^\T A (X - B) \leq  C \}.\label{Eq:ellipsoid}
\end{align}
The projection onto $\mathbb{E}$ is derived as the following:

\begin{lemma}[Ellipsoid Projection]
Suppose that $\text{\bf Angle} (\cdot)$ is the angle function and  $\text{\bf Ramp}(\cdot)$ is a Ramp function for a symmetric matrix, defined as
\begin{align}
\text{\bf Angle} (X) &\triangleq  X (X^\T X)^{-\frac{1}{2}},\\
    \text{\bf Ramp}(A) &\triangleq U^\T \text{\bf diag}(\text{\bf R}(\lambda_1),\text{\bf R}(\lambda_2),\ldots, \text{\bf R}(\lambda_n)) U\\
    \text{\bf R}(x) &\triangleq \begin{cases}
        x &x\geq 0\\
        0 &x< 0.
    \end{cases}
\end{align}
where $U$ is the eigenbasis matrix of $A$.
The following map is a projection onto the ellipsoid $\mathbb{E}$:
\begin{align}
&P_{\mathbb{E}} (X) \triangleq B+ A^{-\frac{1}{2}}\text{\bf Angle}(\sqrt{A}(X - B)) \nonumber\\
&\hspace{1ex}\cdot \sqrt{C - \text{\bf Ramp}(C - (X - B)^\T A (X - B))} .
\end{align}
\end{lemma}
\begin{proof}
The conditions for a mapping $P_{\mathbb{E}}$ to be a projection onto $\mathbb{E}$ are: 
\begin{enumerate}[(i)]
    \item $\Im(P_{\mathbb{E}}) = \mathbb{E}$,
    \item $P_{\mathbb{E}}\circ P_{\mathbb{E}} = P_{\mathbb{E}}$.
\end{enumerate}
We show that $P_\mathbb{E}$ satisfies these conditions.

{\bf Condition (i) :}
For any $X \in \R^{n\times m}$, the $\text{\bf Angle}(X)$ is an orthogonal matrix, written as
\begin{align*}
    &\text{\bf Angle}(X)^\T \text{\bf Angle}(X) \\
    &= \Big( X(X^\T X)^{-\frac{1}{2}}\Big)^\T\Big( X(X^\T X)^{-\frac{1}{2}}\Big)\\
    &= (X^\T X)^{-\frac{1}{2}}  X^\T X (X^\T X)^{-\frac{1}{2}}\\
    &=I_m.
\end{align*}
The projected matrix $P_\mathbb{E}(X)$ satisfies
\begin{align}
     &(P_\mathbb{E}(X) - B)^\T A (P_\mathbb{E}(X) - B)\nonumber\\
     &=  \Big(A^{-\frac{1}{2}}\text{\bf Angle}(\sqrt{A}(X - B)) \nonumber\\
     &\hspace{2ex}\cdot \sqrt{C - \text{\bf Ramp}(C - (X - B)^\T A (X - B))}\Big)^\T A \Big( * \Big)\nonumber\\
     &=C - \text{\bf Ramp}(C - (X - B)^\T A (X - B)) \label{Eq:proof_ellipsoid_projection_1}\\
     &\leq C.\nonumber
\end{align}
Therefore, $\Im(P_\mathbb{E}) \subseteq \mathbb{E}$ for all $X\in \R^{n\times m}$. 
Assuming $X\in \mathbb{E}$, the output of this projection $P_\mathbb{E}(X)$ is written as
\begin{align*}
    &P_{\mathbb{E}} (X) =  B+ A^{-\frac{1}{2}}\text{\bf Angle}(\sqrt{A}(X - B))\\
&\hspace{1ex}\cdot \sqrt{C - \text{\bf Ramp}(C - (X - B)^\T A (X - B))}\\
&=  B+ A^{-\frac{1}{2}}\text{\bf Angle}(\sqrt{A}(X - B))\\
&\hspace{1ex}\cdot \sqrt{ (X - B)^\T A (X - B)}\\
&=B+ A^{-\frac{1}{2}}(\sqrt{A}(X - B))\\
&= X.
\end{align*}
Hence, $\Im(P_\mathbb{E}) = \mathbb{E}$.

{\bf Condition (ii) :} 
For any $X\in \R^{n\times m}$, the self-composition of the mapping $P_\mathbb{E}$ satisfies the following:
\begin{align*}
    &P_\mathbb{E} \circ P_\mathbb{E}(X) \\
    &= B+A^{-\frac{1}{2}}\text{\bf Angle}(\sqrt{A}(P_\mathbb{E}(X) - B))\\
    &\hspace{0ex} \cdot \sqrt{C - \text{\bf Ramp}(C - (P_\mathbb{E}(X) - B)^\T A (P_\mathbb{E}(X) - B))} .
\end{align*}
The internal of the $\text{\bf Angle}(\cdot)$ function satisfies
\begin{align*}
    &\sqrt{A}(P_\mathbb{E}(X) - B)\\
    &= \sqrt{A}\Big(B + A^{-\frac{1}{2}}\text{\bf Angle}(\sqrt{A}(X - B)) \\
    &\cdot\sqrt{C - \text{\bf Ramp}(C - (X - B)^\T A (X - B))}   - B\Big)\\
    &= \text{\bf Angle}(\sqrt{A}(X - B)) \\
    &\cdot\sqrt{C - \text{\bf Ramp}(C - (X - B)^\T A (X - B))}.\\
    &\triangleq  \text{\bf Angle}(\sqrt{A}(X - B))\sqrt{D}.
\end{align*}
Hence, the angle is written as 
 \begin{align*}
&\text{\bf Angle}(\sqrt{A}(P_\mathbb{E}(X) - B))\\
&=\text{\bf Angle}\Big(\text{\bf Angle}(\sqrt{A}(X - B))\sqrt{D}\Big)\\
&= \text{\bf Angle}(\sqrt{A}(X - B)) \sqrt{D} D^{-\frac{1}{2}}\\
&= \text{\bf Angle}(\sqrt{A}(X - B)) 
\end{align*}
In contrast, the inner of root is written as 
\begin{align*}
&C - \text{\bf Ramp}(C - (P_\mathbb{E}(X) - B)^\T A (P_\mathbb{E}(X) - B))\\
&= C - \big(C - (P_\mathbb{E}(X) - B)^\T A (P_\mathbb{E}(X) - B)\big)\\
&=(P_\mathbb{E}(X) - B)^\T A (P_\mathbb{E}(X) - B)\\
&=  C - \text{\bf Ramp}(C - (X - B)^\T A (X - B))
\end{align*}
using Eq.~(\ref{Eq:proof_ellipsoid_projection_1}).
Therefore, 
\begin{align*}
    &P_\mathbb{E} \circ P_\mathbb{E}(X) \\
    &= B+A^{-\frac{1}{2}}\text{\bf Angle}(\sqrt{A}(P_\mathbb{E}(X) - B))\\
    &\hspace{0ex} \cdot \sqrt{C - \text{\bf Ramp}(C - (P_\mathbb{E}(X) - B)^\T A (P_\mathbb{E}(X) - B))}\\
    &=B + A^{-\frac{1}{2}}\text{\bf Angle}(\sqrt{A}(X - B)) \\
    &\hspace{5ex} \cdot\sqrt{C - \text{\bf Ramp}(C - (X - B)^\T A (X - B))}\\
    &=P_\mathbb{E}(X).
\end{align*}

From conditions (i) and (ii), $P_\mathbb{E}$ is the projection onto the ellipsoid $\mathbb{E}$.
\end{proof}

Using the following lemma, the general dissipative projection is derived.

\begin{theorem}[General Dissipative Projection]
Assume that $\Qs$ is a negative definite matrix and $\Rs- \Ss^\T \Qs^{-1}\Ss$ is a positive semi-definite matrix.
The following map  $\mathcal{P}_{\ls,V}: (\fn,\gn,\hn,\jn) \mapsto (\fm, \gm, \hmm, \jm)$ is a dissipative projection:
\begin{align}
\begin{aligned}
\fm &= \PsC \fn - \text{\bf Angle}(\nabla V)  \big(\|\ls\|^2 - \hn^\T Q \hn\big)\\
\gm &= \PsC \fn  + 2\text{\bf Angle}(\nabla V) \\
&\hspace{5ex}\cdot \big(  \hn^\T \big(S+ Q P_{\rm j}(\jn) \big) - \ls^\T P_{\rm W}(\jn)\big)\\
\hmm &=\hn,\quad \jm =  P_{\rm j}(\jn) 
\end{aligned}\label{Eq:general_dissipative_projection}
\end{align}
where $P_{\rm j}(\jn)$ is the projection of $\jn$ onto the ellipsoid~(\ref{Eq:ellipsoidal_equation}) and  $P_{\rm W}(\jn)$ is $W$ on the nonlinear KYP lemma~(\ref{Eq:KYP_condition}) after projected $\jn$, given by, 
\begin{align*}
P_{\rm j}(\jn) &\triangleq (-Q)^{-\frac{1}{2}} \text{\bf Angle}(\sqrt{-Q}\jn -(- Q)^{-\frac{1}{2}}S) \nonumber\\
 &\hspace{1ex} \cdot \sqrt{R - S^\T Q^{-1}S - P_{\rm W}(\jn)^2} - Q^{-1}S,\\
P_{\rm W}(\jn) &\triangleq \sqrt{\text{\bf Ramp}(\jn^\T Q\jn + S^\T \jn + \jn^\T S + R)}.
\end{align*}
\end{theorem}
\begin{proof}
The freedom degree of $W$, the ellipsoid condition on the general solution of the QME~(\ref{Eq:QME}) is written as
\begin{align*}
    &(\jn + \Qs^{-1}\Ss)^\T(-\Qs)(\jn + \Qs^{-1}\Ss) \\
    &= \Rs - \Ss^\T \Qs^{-1}\Ss - \Ws^\T \Ws\\
    &\leq \Rs - \Ss^\T \Qs^{-1}\Ss.
\end{align*}
Hence, the projected $\jn$ as $\jm$ belong the ellipsoid~(\ref{Eq:ellipsoid}) where each parameter is satisfies
\begin{align*}
    A = -Q,\quad B = -Q^{-1}S,\quad C= R- S^\T Q^{-1}S.
\end{align*}
Therefore, the projection $P_{\rm j}(\cdot)$ is derived the previous lemma.

In contract, the parameter $W$ when $\jn$ is projected $\jm = P_{\rm j}(\jn)$ is given by 
\begin{align*}
&\Ws^\T \Ws  \\
&=  \Rs - \Ss^\T \Qs^{-1}\Ss \\
&\hspace{5ex} -  (P_{\rm j}(\jn) + \Qs^{-1}\Ss)^\T(-\Qs)(P_{\rm j}(\jn) + \Qs^{-1}\Ss)\\
&= P_{\rm W}(\jn)^2
\end{align*}
Hence, $\Ws = U P_{\rm W}(\jn)$ where $U$ is an orthogonal matrix. 
Rewriting $\ls\leftarrow U^\T \ls$, the general dissipative projection~(\ref{Eq:general_dissipative_projection}) is derived from the general solution of the QME~(\ref{Eq:QME}).
It is easy to check that condition $\mathcal{P}_{\ls,V}\circ \mathcal{P}_{\ls,V} = \mathcal{P}_{\ls,V}$ holds, so we omit it.
\end{proof}
\section{Proof of Dissipative Projection}\label{APP:Proof_of_main_theorem}
Assuming $\jh = \js = 0$, the ellipse equation is written as
\begin{align*}
     &(\jh +  \Qs^{-1}\Ss)^\T (-\Qs)(\jh +  \Qs^{-1}\Ss)\\
     &=  - \Ss^\T \Qs^{-1}\Ss  \\
     &=  \Rs - \Ss^\T \Qs^{-1}\Ss - \Ws^\T \Ws.
\end{align*}
Hence $\Rs =  \Ws^\T \Ws$. 
Considering an orthogonal matrix $U$, the parameter $\Ws$ satisfies $\Ws  = U \sqrt{R}$.
Therefore, the each QME solution (\ref{Eq:QME_solution}) satisfies
\begin{align*}
\fs&=  \PsC \fh + \frac{\nabla V}{\|\nabla V\|^2} \Big( \hh^\T \Qs \hh - \|\ls\|^2\Big)\\
&=\PsC \fh + \frac{\nabla V}{\|\nabla V\|^2} \Big( \hh^\T \Qs \hh - \|\tilde{\ls}\|^2\Big)\\
\gs &= \PsC \gh + 2\frac{\nabla V}{\|\nabla V\|^2} \Big(\hh^\T ( \Ss + \Qs \jh) - \ls^\T \Ws \Big)\\
&= \PsC \gh + 2\frac{\nabla V}{\|\nabla V\|^2} \Big(\hh^\T \Ss - \ls^\T U\sqrt{R}\ \Big)\\
&= \PsC \gh + 2\frac{\nabla V}{\|\nabla V\|^2} \Big(\hh^\T \Ss - \tilde{\ls}^\T\sqrt{R}\ \Big)\\
\hs&=\hh
\end{align*}
where $\tilde{\ls}\triangleq U^\T \ls$.
Therefore, if we denote the mapping from $(\fh,\gh,\hh)$ to $(\fs,\gs,\hs)$ as $\mathcal{P}_{\ls,V}$, the image is belong to dissipative dynamics $\mathbb{S}_d$ for any parameter $\ls,V$, that is,
\begin{align}
    \Im (\mathcal{P}_{\ls,V}) \triangleq \mathbb{S}_{\ls,V} \subset \mathbb{S}_d.\label{Eq:proof_dissipative_projection_condition_1}
\end{align}
Furthermore, since Eq.~(\ref{Eq:QME_solution}) is the general solution of Eq.~(\ref{Eq:KYP_condition}) in the non-linear KYP lemma, the sum set of the image $\mathcal{P}_{\ls,V}$ equals the entire dynamics satisfying dissipativity, that is 
\begin{align}
    \bigcup_{l,V} \mathbb{S}_{\ls,V} = \mathbb{S}_d.\label{Eq:proof_dissipative_projection_condition_2}
\end{align}

For all $(\fs_1,\gs_1,\hs_1)$, we define the projected values is defined as 
\begin{align*}
    (\fs_2,\gs_2,\hs_2)&\triangleq \mathcal{P}_{\ls,V}(\fs_1,\gs_1,\hs_1),\\
    (\fs_3,\gs_3,\hs_3)&\triangleq \mathcal{P}_{\ls,V}(\fs_2,\gs_2,\hs_2).
\end{align*}
The each element satisfies the following:
\begin{align*}
\fs_2 &=\PsC \fs_1 + \frac{\nabla V}{\|\nabla V\|^2} \Big( \hs_1^\T \Qs \hs_1 - \|\tilde{\ls}\|^2\Big)\\
\gs_2 &= \PsC \gs_1 + 2\frac{\nabla V}{\|\nabla V\|^2} \Big(\hs_1^\T \Ss - \tilde{\ls}^\T\sqrt{R}\ \Big)\\
\hs_2&= \hs_1,\\
\end{align*}
\begin{align*}
\fs_3 &=\PsC \fs_2 + \frac{\nabla V}{\|\nabla V\|^2} \Big( \hs_2^\T \Qs \hs_2 - \|\tilde{\ls}\|^2\Big)\\
&=\PsC \Big(\PsC \fs_1 + \frac{\nabla V}{\|\nabla V\|^2} \big( \hs_1^\T \Qs \hs_1 - \|\tilde{\ls}\|^2\big)\Big) \\
&\hspace{15ex}+\frac{\nabla V}{\|\nabla V\|^2} \Big( \hs_1^\T \Qs \hs_1 - \|\tilde{\ls}\|^2\Big)\\
&=\PsC \circ \PsC \fs_1 + \frac{\nabla V}{\|\nabla V\|^2} \Big( \hs_1^\T \Qs \hs_1 - \|\tilde{\ls}\|^2\Big)\\
&=\PsC \fs_1 + \frac{\nabla V}{\|\nabla V\|^2} \Big( \hs_1^\T \Qs \hs_1 - \|\tilde{\ls}\|^2\Big)\\
&= \fs_2\\
\gs_3 &= \PsC \gs_2 + 2\frac{\nabla V}{\|\nabla V\|^2} \Big(\hs_2^\T \Ss - \tilde{\ls}^\T\sqrt{R}\ \Big)\\
&= \PsC \circ \PsC \gs_1 + 2\frac{\nabla V}{\|\nabla V\|^2} \Big(\hs_1^\T \Ss - \tilde{\ls}^\T\sqrt{R}\ \Big)\\
&= \gs_2\\
\hs_3&= \hs_2.
\end{align*}
In some equation variants, the following were used: 
\begin{align*}
    \PsC \nabla V &= 0,\quad  \PsC \circ \PsC = \PsC.
\end{align*}
Hence, 
\begin{align}
    \mathcal{P}_{\ls,V} \circ \mathcal{P}_{\ls,V} = \mathcal{P}_{\ls,V}.\label{Eq:proof_dissipative_projection_condition_3}
\end{align}

Therefore, from three conditions (\ref{Eq:proof_dissipative_projection_condition_1}), (\ref{Eq:proof_dissipative_projection_condition_2}), and (\ref{Eq:proof_dissipative_projection_condition_3}), $\mathcal{P}_{\ls,V}$ is a dissipative projection. \hspace{\fill}$\Box$

\section{Proof of Energy Conservation Projection}\label{APP:Proof_of_Energy_Conservation_Projection}
If $R=0$ and $\ls \equiv 0$, the dissipative projection is transformed into equation~(\ref{Eq:energy_conservation_projection}).
Also,  $W=U\sqrt{R} = 0$.
From equation~(\ref{Eq:non_linear_KYP_lemma_V_dot}) on the proof of nonlinear KYP lemma, the difference of the time derivative of storage function $V(x)$ and supply rate $w(\us,\ys)$ is given by
\begin{align*}
\dot{V}(x)  -   w(u,y) &=  \|\ls(x) + \Ws(x) u\|^2 = 0.
\end{align*}
Therefore, this energy conservation projection satisfies the equivalent condition of this integral.
\hspace{\fill}$\Box$
\section{Passivity and its projection}\label{APP:Passivity}
Passivity can be described as the exchange of energy defined as the inner product of input signals $u$ and output signals $y$. 
It is described more simply than dissipativity, assuming that the dimensions of the input and output are equal.

\begin{definition}[Passivity]
    If the input-output system~(\ref{Eq:main_system}) without direct-path is passive, that is
\begin{align*}
    \inner<u|y> \triangleq\int_0^\infty u^\T(t) y(t) \dd t \geq 0.
\end{align*}
\end{definition}

Compared to dissipativity, it corresponds to when the parameter of supply ratio is $Q=R=0,~S=\frac{1}{2}I_m$.
Note that passivity assumes the initial condition to be at the equilibrium point, allowing the exclusion of the storage function.

Then, the Lur'e equation is proposed as a method to determine passivity.

\begin{corollary}[Nonlinear Lur'e Equation]
If there exists a differentiable function $V:\R^n \rightarrow \R$ satisfies the following condition:
\begin{align}
\begin{aligned}
\nabla V^\T(x) f(x) &\leq 0, \\
\nabla V^\T(x) g(x) &= 2h^\T (x),     
\end{aligned} \label{coro:passivity}
\end{align}
then the input-output system~(\ref{Eq:main_system}) is passive.
\end{corollary}
This corollary can be easily demonstrated by designing the parameters of the supply ratio into the nonlinear KYP lemma.

Furthermore, as with internal stability and input-output stability, projections to dynamics that guarantee passivity can be easily derived using the theorem of dissipation projection(Theorem~\ref{Thm:Dissipative_Projection}).
\begin{corollary}[Passive Projection] \label{Cor:Proj_Passivity}
The following mapping $\mathcal{P}_{V} : (\fn,\gn,\hn ) \mapsto (\fm,\gm,\hmm)$:
\begin{align}
\begin{aligned}
\fm &= \fn  -  \frac{\nabla V}{\|\nabla V\|^2}\mathrm{ReLU}(\nabla V^\T \fn),\\
\gm &= \gn  -  \frac{\nabla V}{\|\nabla V\|^2} \big(\nabla V^\T \gn  - 2\hn^\T \big),\\
\hmm &= \hn 
\end{aligned}
\label{Eq:Projection_Passivity}
\end{align}
is a projection into passive dynamics.
\end{corollary}
\begin{proof}
    When $S=\frac{1}{2}I$, $R=Q=0$, it is derived from the theorem.
\end{proof}

Passivity dynamics can be represented as linear constraints, such as in Lur'e's lemma, allowing for the precise definition of metrics and orthogonal projections.
Therefore, the relationship between projection and metric is simpler than that of a dissipative projection and is closed with a discussion of Hilbert spaces.
In the next section (Appendix~\ref{APP:Another_Dissipative_Projection}), we focus on examining the relationship between metrics and projections for the dynamics $(f, g, h)$, with an emphasis on passivity.
\section{Dissipative projection and Metric Space}\label{APP:Another_Dissipative_Projection}
In this chapter, we examine the effects of dissipative projections that result from metric transformations. 
Since the passivity condition is a linear constraint, it naturally leads to the introduction of orthogonal projections in a Hilbert space.
Therefore, in the following, we present the results of metric transformations with respect to passivity.

First, we consider the case where a simple 2-norm is used as the metric for $(f, g, h)$.
\begin{corollary}
Consider the following optimization problem:
\begin{align}
\begin{aligned}
\underset{ \fm,\gm,\hmm }{\text{\bf minimize}}&\quad \|\fm - \fn\|^2 + \|\gm - \gn\|^2 + \|\hmm - \hn\|^2 \\
\text{\bf subject to}  &\quad (\fm,\gm,\hmm)~\text{\rm satisfies passivity. }
\end{aligned}
\end{align}
The solution is the following:
\begin{subequations}
\begin{align}
    \fm &= \fn - \frac{1}{\|\nabla V\|^2} \mathrm{ReLU}(\nabla V^\T  \fn ) \nabla V \label{Eq:proj_passivity_f}\\
    \gm &= \gn -  \frac{1}{4 + \|\nabla V\|^2} \nabla V (\nabla V^\T \gn  - 2\hn^\T) \label{Eq:proj_passivity_g}\\
    \hmm^\T &= \hn^\T + \frac{2}{4 + \|\nabla V\|^2} (\nabla V^\T \gn  - 2\hn^\T),\label{Eq:proj_passivity_h}
\end{align}\label{Eq:proj_passivity}
\end{subequations}
where, this projection is denoted as $\mathcal{P}_{V}^\alpha$.
\end{corollary}
\begin{proof}
Since Lur'e equation~(\ref{coro:passivity}) is a linear about $(\fn,\gn,\hn)$, Eq.~(\ref{Eq:proj_passivity}) is a quadratic programming (QP) problem and a unique analytic solution exists.
Furthermore, the Lur'e equation are independent of $\fn$ and $(\gn, \hn)$.
The optimization problem~(\ref{Eq:proj_passivity}) can be solved as an optimization problem for variables $\fm$ and $[\gm, \hmm^\T ]$, respectively.

The optimization problem of $\fm$ matches the internal stability condition in this study~\cite[Theorem 1]{Manek2019}.
Hence, the solution of $\fm$ is derive as the solution~(\ref{Eq:proj_passivity_f}).

In contrast, the optimization problem of $(\gm, \hmm)$ is written as 
\begin{align*}
\underset{\gm,\hmm}{\text{\bf minimize}} &\quad \left\| \begin{bmatrix}\gm \\ \hmm^\T\end{bmatrix} - \begin{bmatrix}\gn \\ \hn^\T\end{bmatrix} \right\|^2\\
\text{\bf subject to} &\quad 
[\nabla V^\T~ -2]
\begin{bmatrix}\gm \\ \hmm^\T\end{bmatrix}
 = 0.
\end{align*}
The solution to this QP problem is derived as
\begin{align*}
&\begin{bmatrix}\gm \\ \hmm^\T\end{bmatrix} -  \begin{bmatrix}\gn \\ \hn^\T\end{bmatrix}\\
&=- \frac{1}{\| [\nabla V^\T,~ -2]\|^2} \begin{bmatrix}\nabla V \\ -2 \end{bmatrix} [\nabla V^\T~ -2]  
\begin{bmatrix}\gn \\ \hn^\T\end{bmatrix}\\
&=- \frac{1}{4+\|\nabla V \|^2} \begin{bmatrix}\nabla V \\ -2 \end{bmatrix} 
[\nabla V^\T \gn  - 2\hn^\T].
\end{align*}
Therefore, the solutions $\gm$ and $\hmm$ is written as Eq.~(\ref{Eq:proj_passivity_g}) and (\ref{Eq:proj_passivity_h}), respectively.
\end{proof}

The metric of the inner product that derived the orthogonal projection $\mathcal{P}_{V}^\alpha$ in the lemma takes the following simple form:
\begin{align}
&\inner<f_a,g_a,h_a| f_b,g_b,h_b>_\alpha \nonumber\\
&\triangleq f_a^\T f_b + \text{Tr}\left(\begin{bmatrix}g_a^\T  &h_a\end{bmatrix} \begin{bmatrix} g_b \\ h_b^\T\end{bmatrix} \right) \label{Eq:innner_product_normal}
\end{align}
where $\text{Tr}(A)$ is the trace of the square matrix $A$. 
Note that the integral over the domain $\R^n$ is omitted.
Since orthogonality of $f$ is trivial and independent, considering the orthogonality of the projection with respect to $g$ and $h$ yields the following:
\begin{align*}
    &\inner< \mathcal{P}_{V}^\alpha(g_a,h_a) | g_b,h_b >_\alpha\\
    &= \text{Tr} \Big( \left( \begin{bmatrix}  g_a \\ h_a^\T \end{bmatrix} - \frac{1}{4+\|\nabla V\|^2} \begin{bmatrix}  \nabla V(\nabla V^\T g_a - 2h_a^\T) \\ - 2\nabla V^\T g_a +4h_a^\T \end{bmatrix} \right)^\T \\
    & \hspace{45ex}\begin{bmatrix} g_b \\ h_b^\T\end{bmatrix} \Big) \\
    &=  \inner< g_a,h_a | g_b,h_b >_\alpha 
     \\
    &\hspace{1ex}- \frac{1}{4 + \|\nabla V\|^2}\text{Tr }\left( (\nabla V^\T g_a - 2h_a^\T)^\T (\nabla V^\T g_bz - 2 h^\T_b) \right)\\
    &= \inner< g_a,h_a | \mathcal{P}_{V}^\alpha(g_b,h_b) >_\alpha.
\end{align*}
Hence, $\mathcal{P}_{V}^\alpha$ is an orthogonal projection with respect to this inner product  $\inner<\cdot|\cdot>_\alpha$.

In contrast, the passivity projection defined on the main paper~(Corollary~\ref{coro:passivity}), denoted as $\mathcal{P}_{V}^\beta$ is also an orthogonal projection with a different inner product.
Here, we define a inner product about $(f,g,h)$ as
\begin{align}
    &\inner< f_a,g_a,h_a|f_b,g_b,h_b>_\beta \nonumber \\
    &\triangleq f_a^\T f_b + \text{Tr}\left( \begin{bmatrix}g_a^\T  &h_a\end{bmatrix}
    \begin{bmatrix}  \|\nabla V\|^2 I_m & - 2\nabla V \\ - 2\nabla V^\T& 4+\varepsilon \end{bmatrix}
    \begin{bmatrix} g_b \\ h_b^\T\end{bmatrix} \right).
\end{align}
If $\varepsilon >0$, $\inner<\cdot|\cdot>_\beta$ satisfies the axioms of the inner product.
Therefore, the fact that $\mathcal{P}_{V}^\beta$ is an orthogonal projection can be confirmed as follows.
\begin{align*}
    &\inner<\mathcal{P}_{V}^\beta(g_a,h_a) | g_b,h_b >_\beta \\ 
    &= \text{Tr}\Bigg( \left( \begin{bmatrix}  g_a \\ h_a^\T \end{bmatrix} - \frac{1}{\|\nabla V\|^2}\begin{bmatrix}  \nabla V(\nabla V^\T g_a - 2h_a^\T) \\ 0 \end{bmatrix} \right)^\T\\
    &\hspace{22ex}       \begin{bmatrix}  \|\nabla V\|^2 I_m & - 2\nabla V \\ - 2\nabla V^\T& 4+\varepsilon \end{bmatrix}\begin{bmatrix} g_b \\ h_b^\T\end{bmatrix} \Bigg) \\
    &= \inner<g_a,h_a | g_b,h_b >_\beta \\
    &\hspace{10ex}-\text{Tr}\Big( \big(\nabla V^\T g_a - 2h_a^\T\big)^\T \big(\nabla V^\T g_b - 2h_b^\T\big)\Big) \\
    &= \inner<g_a,h_b|\mathcal{P}_{V}^\beta(g_b,h_b)>_\beta. 
\end{align*}
Through the above, it has been confirmed that $P^\alpha$ and $P^\beta$ are orthogonal projections of different inner products.
Similarly, if we define a different inner product, we can define a different passive projection.

In learning prediction errors, the optimal choice of inner product and passive projection is not clear.
For example, as $\mathcal{P}^\beta$ is an identity map for $h$ in terms of passive projection, learning updates for $g$ becomes relatively challenging.
Designing the optimal inner product and gradient method for learning using parametric projection is our next challenge.

Furthermore, as the general dissipativity condition are nonlinear constraints (namely, the QME~(\ref{Eq:QME})), it is challenging to generalize this discussion to dissipative projections.
Dissipative projections are not linear projections naturally defined on a Hilbert space, so they cannot be defined as orthogonal projections using an inner product.
The optimal dissipative projection on non-Hilbert spaces and how to learn it are the issues.
\section{Difference of our Previous Study}\label{App:Difference_of_our_Previous Study}
Dissipativity deals with generalized time series input-output characteristics, and various input-output characteristics can be expressed depending on the setting of the supply rate $w(u,y)$.
Therefore, the input-output stability treated in our previous study \cite{kojima2022learning} is one example, and other input-output characteristics such as internal stability, passivity, and energy conservation can also be handled.
Furthermore, while the Hamilton-Jacobi inequality in our previous study is a sufficient condition, the nonlinear KYP lemma in this study is a necessary and sufficient condition. 
Therefore, the range of dynamics explored by this method is wider. Unlike our previous study, the number of NNs to be trained increases because an indirect vector function $l(x)$ is introduced.

\section{Detail of Experiments}\label{APP:Examples}
This section shows the detail examples results and the setting of the experiments.
Furthermore,  we explain the physical information of target model, based on energy properties.

\subsection{Mass-Spring-Damper(Linear) Dynamics} \label{APP:exaple_linear}

The mass-spring-damper system is one of the most commonly discussed examples of linear systems and provides important insights as a time-series system (See Figure~\ref{fig:mass-spring-damper}).
Letting $m$  be the mass, $k$ the spring constant, $c$ the damping coefficient, and $F$ the external force, the mass-spring-damper system is described by the following dynamics:
\begin{align*}
m \ddot{q} + c\dot{q} + k{q} = F,\quad q(0) = 0,
\end{align*}
where,  $q$ represents the position of the point mass. 
Since the mass-spring-damper system is a dynamical system, its energy relation is well-defined, and the following equation holds:
\begin{align*}
&\underbrace{\frac{1}{2}kq^2(T)}_{\text{Potential Energy}} + \underbrace{\frac{1}{2}m\dot{q}^2(T)}_{\text{Kinetic Energy}}\\&\hspace{10ex}+\underbrace{\int_{0}^T c \dot{q}^2(t)\dd t}_{\text{Dissipation Term}}
= \underbrace{\int_{0}^T F(t)\dot{q}(t)\dd t}_{\text{Work}}.
\end{align*}
The sum of the first and the second term represents the internal energy of this system.
The third term represents the energy dissipation due to the damper, and the final term indicates the energy input from the external force $F$.

\begin{figure}[t]
    \centering
     \includegraphics[width=0.8\linewidth]{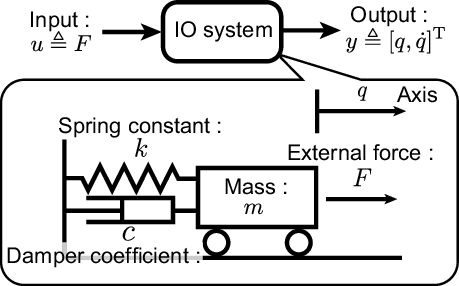}
    \caption{The sketch of the mass-spring-damper model.}
    \label{fig:mass-spring-damper}
\end{figure}

Consider the external force $F$ as the input, the position $q$ and velocity $\dot{q}$ as the outputs, written as
\begin{align*}
x \triangleq \begin{bmatrix}
q\\\dot{q}
\end{bmatrix},\quad 
u \triangleq F,\quad y \triangleq 
\begin{bmatrix}
q\\\dot{q}
\end{bmatrix}.
\end{align*}
The state-space model can be expressed as follows:
\begin{align*}
\dot{x} = \begin{bmatrix}
0&1\\
- \tfrac{k}{m}& - \tfrac{c}{m}
\end{bmatrix}x + 
\begin{bmatrix}
    0\\
    \frac{1}{m}
\end{bmatrix}
u, \quad y = x.
\end{align*}
Then, defining the storage function and supply rate for dissipativity as:
\begin{align*}
V(x) &\triangleq \frac{1}{2}kx_1^2  + \frac{1}{2}m x_2^2,\\
w (u,y) &\triangleq  [y~u]
\begin{bmatrix}
    0&0&0\\
    0&-c&\tfrac{1}{2}\\
    0&\tfrac{1}{2}&0
\end{bmatrix}
\begin{bmatrix}
y\\
u
\end{bmatrix},
\end{align*}
the condition of dissipativity (\ref{Eq:dissipativity}) is satisfied.
Therefore, the mass-spring-damper model satisfies the dissipativity with quadratic supply rate defined as 
\begin{align*}
    Q= \begin{bmatrix}
        0&0\\0&-c
    \end{bmatrix},
    S= \begin{bmatrix}
        0\\\frac{1}{2}
    \end{bmatrix},\quad
    R = 0.
\end{align*}

Furthermore, it is guaranteed that in this model, the origin $x \equiv 0$  is asymptotically stable, and from the perspective of energy conservation, the input-output stability is also ensured.

\begin{table}[t]
    \centering
    \begin{tabular}{|c||r|r|}
    \hline 
     Data& model& Time(hour) \\ \hline 
     \hline 
     \multirow{5}{*}{
     \begin{tabular}{c}
          N=100
     \end{tabular}}
&Naive	&$1.35	\pm0.26$ \\ \cline{2-3}
&Stable	&$1.59	\pm0.21$ \\ \cline{2-3}
&Conservation	&$2.44	\pm0.61$ \\ \cline{2-3}
&IO stable	&$2.23	\pm0.44$ \\ \cline{2-3}
&Dissipative	&$2.29	\pm0.53$ \\ \hline
     \multirow{5}{*}{
     \begin{tabular}{c}
          N=1000
     \end{tabular}}
&Naive	&$7.29	\pm1.42$ \\ \cline{2-3}
&Stable	&$11.55	\pm3.10$ \\ \cline{2-3}
&IO stable	&$15.86	\pm5.24$ \\ \cline{2-3}
&Conservation	&$25.61	\pm4.75$ \\ \cline{2-3}
&Dissipative	&$19.36	\pm4.89$ \\ \hline
     \end{tabular}
    \caption{Computational time of the mass-spring-damper benchmark.}
    \label{tab:result_linear_time}
\end{table}

\subsubsection{Experimental Setting :}

In the experiment in Section~\ref{Sec:linear_result}, we use the mass-spring-damper system with $k=m=c=1$ to set up the experiment, including generating the experimental dataset.
To construct the dataset, we prepare input rectangle waves generated by randomly selecting an amplitude of $\pm 1$ and inputting a rectangle wave of random length, with different wave forms for training and testing.
To evaluate robustness to inputs that were not anticipated during training phase, we also showed two results: one using a step input with an amplitude $1.0$ and the other using a waveform generated by a random walk model with a Gaussian distribution of variance 0.005.

The input and output signals on the period $[0,10]$ are sampled with an interval $\Delta t = 0.1$.

The computational time corresponding to Table~\ref{tab:result_linear}
 required for one trial of 5000-epoch training is listed in Table~\ref{tab:result_linear_time}.
For training each method with neural networks, an NVIDIA Tesla A100 GPU was used.

Related to Table~\ref{tab:result_linear},
the difference in the mean values between the naive and other proposed methods was tested using a Bonferroni-corrected t-test with a significance level of 5\%.
As a result, significant differences were observed in the results for all three inputs of conservation with $N = 100$, and no improvements were observed in the others.

To confirm that the dissipativity property is satisfied, we checked the behavior of the dissipative model trained from the mass-spring-damper dynamics by focusing on the storage function $V(x(t))$ and the time integral of the supply rate $w(u(t),y(t))$.
From Figure~\ref{fig:linear_traj_test}, it can be confirmed that the dissipativity inequality~(\ref{Eq:dissipativity}) is always satisfied.
Note that to strictly satisfy the dissipative, the projection must be performed accurately in actual numerical experiments, and numerical errors and division by zero must be prevented.

\begin{figure}[t]
    \centering
     \includegraphics[width=0.9\linewidth]{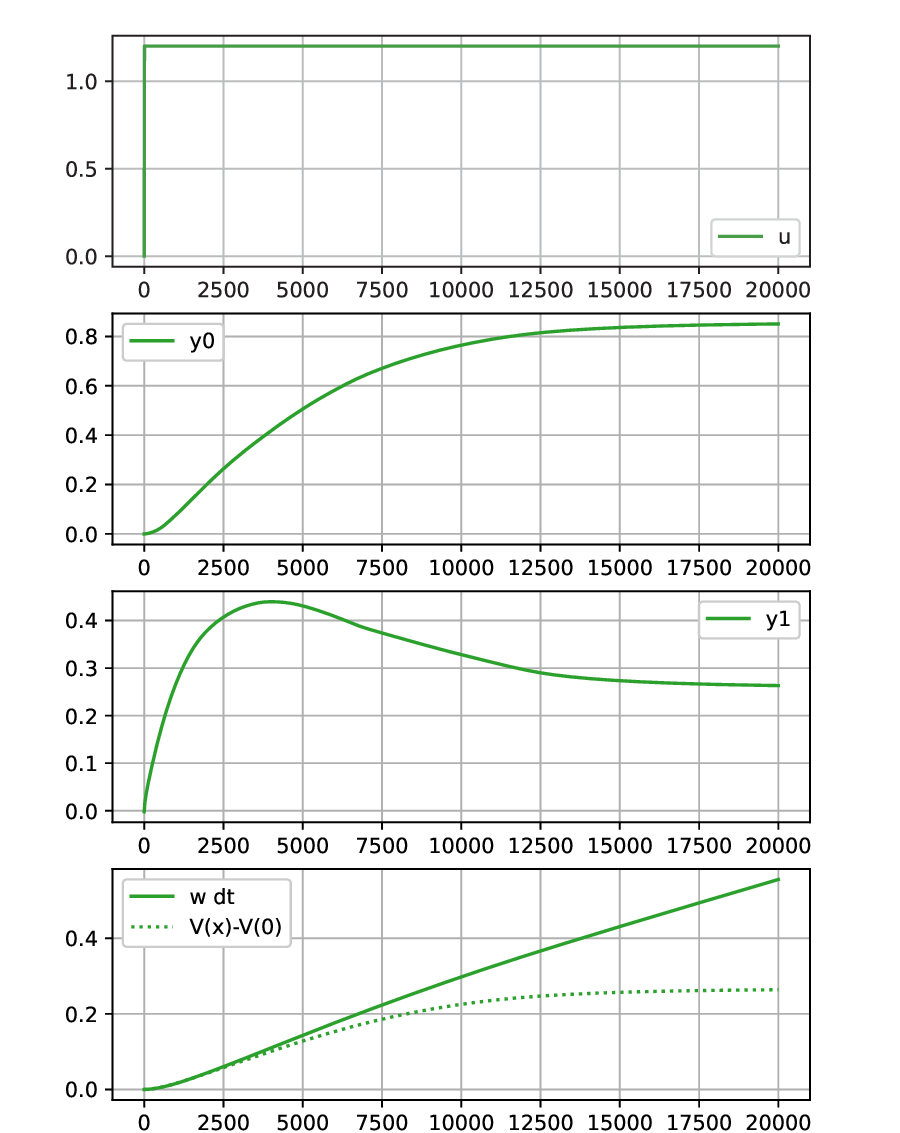}
    \caption{The behaviors of storage function $V(x(t))$ and the time integral of supply rate $w(u(t),y(t))$ in mass-spring-damper dynamics. 
    The top figure is the input $u(t)$, the second and third figures correspond to each dimension of the output $y(t)$, and the bottom figure is the storage function $V(x(t))-V(x(0))$ (dotted line) and the time integral of supply rate $w(u(t),y(t))$ (solid line).
    }
    \label{fig:linear_traj_test}
\end{figure}

\subsection{$n$-link Pendulum}\label{APP:exaple_nlink}
\begin{table*}[t]
    \centering
    \begin{tabular}{|c|c||r|r|r|r|r|}
    \hline 
     Train& Test& Naive  & Stable&  IO stable& Conservation$^*$ & Dissipative \\ \hline 
     \hline 
     \multirow{3}{*}{\begin{tabular}{c}
          Rectangle \\2-link,N=100
     \end{tabular}}&Rectangle &
                    $0.093 \pm 0.058$ & ${\bf 0.079} \pm 0.055$ & $0.147 \pm 0.043$ & $0.163 \pm 0.043$ & $0.105 \pm 0.063$ \\ \cline{2-7}
     &Step        &   $0.069 \pm 0.060$ & ${\bf 0.061} \pm 0.057$ & $0.114 \pm 0.042$ & $0.134 \pm 0.052$ & $0.083 \pm 0.067$ \\\cline{2-7}
     &Random      & ${\bf 0.140} \pm 0.012$ & $0.145 \pm 0.013$ & $0.148 \pm 0.016$ & $0.152 \pm 0.028$ & $0.150 \pm 0.009$ \\\hline

     \multirow{3}{*}{\begin{tabular}{c}Rectangle \\3-link, N=100\end{tabular}}&Rectangle &
                          ${\bf 0.162} \pm 0.077$ & $0.162 \pm 0.078$ & $0.198 \pm 0.068$ & $0.165 \pm 0.100$ & $0.186 \pm 0.075$ \\ \cline{2-7}
     &Step&  $0.267 \pm 0.011$ & $0.264 \pm 0.010$ & $0.307 \pm 0.084$ & $0.251 \pm 0.024 \dagger$ & ${\bf 0.249} \pm 0.042$ \\\cline{2-7}
     &Random&  $0.170 \pm 0.026$ & $0.177 \pm 0.024$ & $0.181 \pm 0.033$ & $0.192 \pm 0.044$ & ${\bf 0.167} \pm 0.054$ \\\hline
     
     \multirow{3}{*}{\begin{tabular}{c}Rectangle\\ 2-link, N=1000\end{tabular}}&Rectangle
                       &  ${\bf 0.038} \pm 0.000$ & $0.039 \pm 0.000$ & $0.084 \pm 0.056$ & ${\bf 0.038} \pm 0.000$ & $0.063 \pm 0.031$ \\ \cline{2-7}
     &Step&  ${\bf 0.021} \pm 0.000$ & $0.020 \pm 0.001$ & $0.063 \pm 0.053$ & ${\bf 0.021} \pm 0.001$ & $0.035 \pm 0.020$ \\\cline{2-7}
     &Random&  $0.156 \pm 0.004$ & ${\bf 0.150} \pm 0.002 \dagger$ & $0.163 \pm 0.023$ & $0.155 \pm 0.004$ & $0.152 \pm 0.008$ \\\hline     

     \multirow{3}{*}{\begin{tabular}{c}Rectangle \\ 3-link, N=1000 \end{tabular}}&Rectangle
                       &  $0.066 \pm 0.002$ & ${\bf 0.064} \pm 0.001$ & $0.114 \pm 0.060$ & $0.065 \pm 0.002$ & $0.064 \pm 0.000 \dagger$ \\ \cline{2-7}
     &Step             &  $0.272 \pm 0.003$ & $0.274 \pm 0.005$ & ${\bf 0.250} \pm 0.029 \dagger$ & $0.273 \pm 0.002$ & $0.272 \pm 0.003$ \\\cline{2-7}
     &Random           &  $0.205 \pm 0.011$ & $0.205 \pm 0.004$ & ${\bf 0.199} \pm 0.013$ & $0.206 \pm 0.006$ & $0.209 \pm 0.002$ \\\hline     
     
     \end{tabular}
    \caption{The prediction error (RMSE) of the  $n$-link pendulum benchmark.}
    \label{tab:result_nlink}
\end{table*}

\begin{figure}[t]
    \centering
     \includegraphics[width=0.8\linewidth]{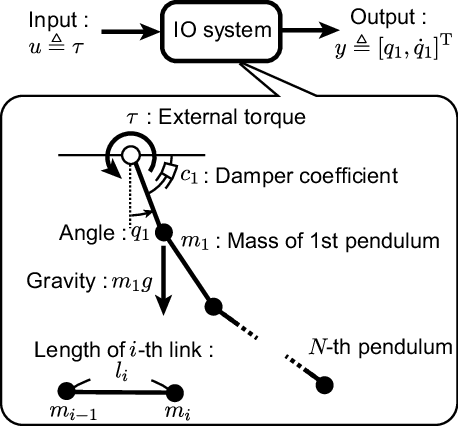}
    \caption{The sketch of $n$-link pendulum.}
    \label{fig:n_link_pendulum}
\end{figure}

\begin{table}[t]
    \centering
    \begin{tabular}{|c||r|r|}
    \hline 
     Data& model& Time(hour) \\ \hline 
     \hline 
     \multirow{5}{*}{
     \begin{tabular}{c}
          2-link,N=100
     \end{tabular}}
&Naive &	$1.59 \pm 	0.16$ \\ \cline{2-3}
&Stable &	$2.25 \pm 	0.25$ \\ \cline{2-3}
&IOstable &	$3.60 \pm 	1.41$ \\ \cline{2-3}
&Conservation &	$2.96 \pm 	0.54$ \\ \cline{2-3}
&Dissipative &	$3.20 \pm 	1.07$ \\ \hline
     \multirow{5}{*}{
     \begin{tabular}{c}
          3-link,N=100
     \end{tabular}}
&Naive	&$12.78	\pm4.74$ \\ \cline{2-3}
&Stable	&$17.33	\pm5.11$ \\ \cline{2-3}
&IOstable	&$29.49	\pm11.03$ \\ \cline{2-3}
&Conservation	&$39.26	\pm9.67$ \\ \cline{2-3}
&Dissipative	&$25.49	\pm12.16$ \\\hline
     \multirow{5}{*}{
     \begin{tabular}{c}
          2-link,N=1000
     \end{tabular}}
&Naive	&$1.55	\pm0.18$ \\ \cline{2-3}
&Stable	&$2.11	\pm0.45$ \\ \cline{2-3}
&IOstable	&$3.07	\pm0.70$ \\ \cline{2-3}
&Conservation	&$4.04	\pm0.79$ \\ \cline{2-3}
&Dissipative	&$2.61	\pm0.36$ \\ \hline
     \multirow{5}{*}{
     \begin{tabular}{c}
          3-link,N=1000
     \end{tabular}}
&Naive	&$7.49	\pm3.61$ \\ \cline{2-3}
&Stable	&$16.51	\pm7.96$ \\ \cline{2-3}
&IOstable	&$24.10	\pm11.07$ \\ \cline{2-3}
&Conservation	&$34.39	\pm16.60$ \\ \cline{2-3}
&Dissipative	&$21.33	\pm11.32$ \\ \hline
     \end{tabular}
    \caption{Computational time of the  $n$-link pendulum benchmark.}
    \label{tab:result_nlink_time}
\end{table}

The $n$-link pendulum is a dynamical model characterized by strong non-linearity and is widely utilized in the evaluation of prediction problems (Figure~\ref{fig:n_link_pendulum}). 
Notably, the behavior of pendulums with $n\geq 2$ exhibits chaotic dynamics, making accurate prediction challenging.
In this study, we employ a model in which dampers, proportional to angular velocity, are applied to each joint of the pendulum, and an external torque is applied to the first joint.
This damped model is commonly used as a representation of a robotic manipulator, making it a model of significant industrial relevance.

Due to the complexity of directly writing out the ordinary differential equation of the $n$-link pendulum, we construct the dynamics using the Euler-Lagrange equation.
Let $q_i$ denote the angle of the $i$-th pendulum and $l_i$ its length of the link.
The coordinates of each pendulum in $\mathbb{R}^2$ are given by:
\begin{align*}
\phi_i &\triangleq l_i \sin(q_i) + \sum_{j=1}^{i-1} l_j \sin(q_j), \\
\psi_i &\triangleq -l_i \cos(q_i) - \sum_{j=1}^{i-1} l_j \cos(q_j), 
\quad i = 1,\ldots, n.
\end{align*}
In this context, the potential energy, kinetic energy, and the energy dissipated by the dampers are defined as follows:
\begin{align*}
P(q) &\triangleq \sum_{i=1}^n m_i g \psi_i, \\
K(q,\dot{q}) &\triangleq \sum_{i=1}^n \frac{1}{2}m_i (\dot{\phi}_i^2 + \dot{\psi}_i^2), \\
D(\dot{q}) &\triangleq \sum_{i=1}^n \frac{1}{2} 
\begin{cases}
c_i\dot{q}_i^2, &i=1,\\
c_i (\dot{q}_i -\dot{q}_{i-1})^2, &i=2,3,\ldots,n.
\end{cases}
\end{align*}
The dissipated energy $D(\dot{q})$ is determined by the relative velocities of each link, where the relative velocity of the first link is given by $\dot{q}_1$ only.
Therefore, the dynamics of the $n$-link pendulum is described using the Euler-Lagrange equation as follows:
\begin{align}
\begin{aligned}
    &\dv{t}\pdv{L}{\dot{q}_i} - \pdv{L}{q_i} + \pdv{D}{\dot{q}_i} = \begin{cases} \tau &i = 1,\\0&i =2,3,\ldots ,n,\end{cases}\\
   &\hspace{25ex}q_i =\dot{q}_i = 0,    
\end{aligned}\label{Eq:n_link_pendulum}
\end{align}
where Lagrangian $L \triangleq K - P$.
Using he potential energy $P(q)$, kinetic energy $K(q,\dot{q})$, and the
energy dissipated $D(\dot{q})$, the energy relation is written as
\begin{align*}
    &\underbrace{P(q(T)) - P(q(0))}_{\text{Potential Energy}} 
    + \underbrace{K(q(T),\dot{q}(T)) - K(q(0),\dot{q}(0))}_{\text{Kinetic Energy}}\\
    &\hspace{20ex}+ \underbrace{\int_{0}^T\inner< \dot{q}|\pdv{D}{\dot{q}}> \dd t}_{\text{Dissipation Term}} 
    =  \underbrace{\int_{0}^T \tau \dot{q}_1 \dd t}_{\text{Work}}.
\end{align*}
Note that the third term can be described as a quadratic form of angular velocities $\dot{q}$, written as
\begin{align*}
\inner< \dot{q}|\pdv{D}{\dot{q}} >&= \dot{q}^\T
\begin{bmatrix}
c_1 \dot{q}_1 + c_2 (\dot{q}_1 - \dot{q}_2)\\
c_2 ( \dot{q}_2 - \dot{q}_1) + c_3  (\dot{q}_2 - \dot{q}_3)\\
\vdots\\
c_{n-1} ( \dot{q}_{n-1} - \dot{q}_{n-2}) + c_n  (\dot{q}_{n-1} - \dot{q}_n)\\
c_n ( \dot{q}_{n} - \dot{q}_{n-1})    
\end{bmatrix}\\
&=  \dot{q}^\T\begin{bmatrix}
    c_1 + c_2& - c_2& 0& \cdots &0\\
    -c_2  &  c_2+ c_3  & - c_3&  &0\\
      0&  -c_3  & c_3 + c_4&   &0\\
      \vdots &&&\ddots&\vdots\\
      0 &0&0&\cdots&c_n
\end{bmatrix}
\begin{bmatrix}
\dot{q}_1\\
\vdots\\
\dot{q}_n
\end{bmatrix}\\
&\triangleq \dot{q}^\T C \dot{q}, 
\end{align*}
where $C$ is a positive definite matrix.
Considering the external torque $\tau$ as the input, the first angle $q_1$ and $\dot{q}_1$ as the velocity, and the all angle and velocity as state, written as  
\begin{align*}
x \triangleq [q_1,\ldots ,q_n,\dot{q}_1,\ldots \dot{q}_n]^\T,\quad 
u \triangleq \tau,\quad 
y \triangleq [q_1,~\dot{q}_1]^\T,
\end{align*}
the dynamics of the state-space is decided as the Euler-Lagrange equation~(\ref{Eq:n_link_pendulum}).

Thereafter, we note that the relationship of dissipativity.
Assuming that the storage function $V$ is written as
\begin{align*}
V(x) \triangleq V(q,\dot{q}) =  P(q) + K(q,\dot{q}),
\end{align*}
the following equation is satisfies:
\begin{align*}
 \dot{V}(x) &=  \tau \dot{q}_1 - \dot{q}^\T C\dot{q}\\
&\leq  \tau \dot{q}_1 - \dot{q}_1^\T c_1\dot{q}_1 \\
&=  \tau [0~1] y  -   y^\T 
\begin{bmatrix}
0&0\\
0&c_1
\end{bmatrix}y\\
&=[y^\T~u^\T]
\begin{bmatrix}
0&0& 0\\
0&-c_1 & \frac{1}{2}\\
0& \frac{1}{2}&0
\end{bmatrix}
\begin{bmatrix}
y\\
u
\end{bmatrix}\\
&\triangleq w(u,y),
\end{align*}
where the above inequality indicates that the total energy dissipation due to the damper is greater than the energy dissipation along the first axis.
Therefore, dissipativity satisfies with the following quadratic supply rate, written as
\begin{align*}
    Q= \begin{bmatrix}
        0&0\\0&-c_1
    \end{bmatrix},
    S= \begin{bmatrix}
        0\\\frac{1}{2}
    \end{bmatrix},\quad
    R = 0.
\end{align*}
Note that the energy conversion projection cannot be made because there are not enough observations to calculate the energy loss of this target.

Furthermore, it is guaranteed that in this model, the origin $x \equiv 0$  is asymptotically stable, and from the perspective of energy conservation, the input-output stability is also ensured.

\subsubsection{Experimental Setting :}

Section~\ref{Sec:nlink_result} uses the n-link pendulum system with $g = 9.81$, $l_i=1/n, m_i=2\frac{n+1}{n}, c_i=1$ for all $i$ to set up the experiment, including generating the experimental dataset.
The input and output signals on the period $[0,1]$ are sampled with an interval $\Delta t = 0.01$.

\subsubsection{Full results :}

Table~\ref{tab:result_nlink} shows the differences in predictive performance of various learning methods when given signal data corresponding to different inputs for the 2-link and 3-link models.
Note that the conservation$*$ results are based on assumptions that are mismatched with the physical model as discussed above.

The computational training time required for one trial of 5000 epochs is listed in Table~\ref{tab:result_nlink_time}.
For training each method with neural networks, an NVIDIA Tesla V100 GPU was used.

Related to Table~\ref{tab:result_nlink},
the difference in the mean values between the naive and other proposed methods was tested using a Bonferroni-corrected t-test with a significance level of 5\%.
As a result, we observed a significant improvement only in cells marked with $\dagger$.

To confirm that the trained dynamics satisfies dissipativity, Figure~\ref{fig:nlink_traj_test} shows the behavior of the dissipative model trained from the 2-link dynamics ($N=1000$) by focusing on the storage function $V(x(t))$ and the time integral of the supply rate $w(u(t),y(t))$.
From the Figure~\ref{fig:nlink_traj_test}, it can be confirmed that the dissipativity inequality~(\ref{Eq:dissipativity}) is always satisfied.
Note that to strictly satisfy the dissipative, the projection must be performed accurately in actual numerical experiments, and numerical errors and division by zero must be prevented.

\begin{figure}[t]
    \centering
     \includegraphics[width=1.0\linewidth]{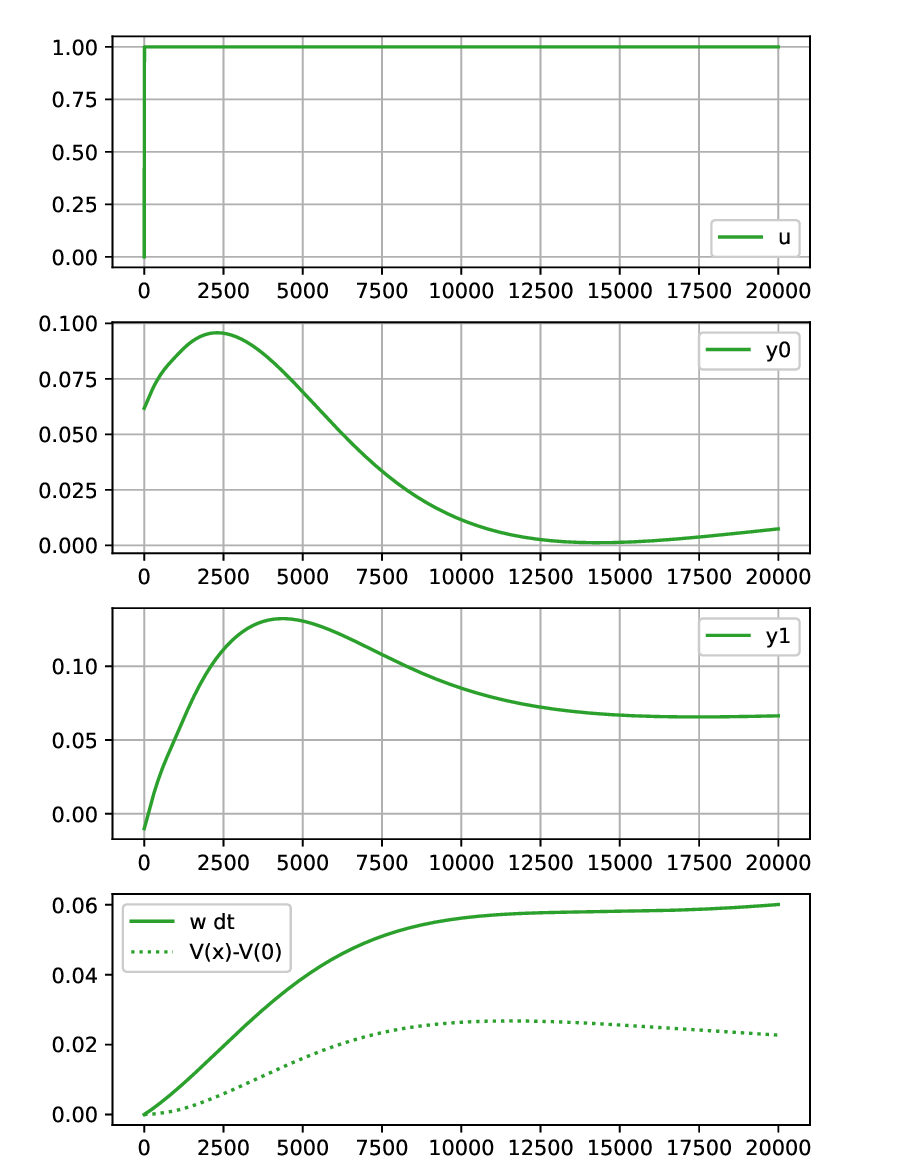}
    \caption{The behaviors of storage function $V(x(t))$ and the time integral of supply rate $w(u(t),y(t))$ in 2-link pendulum. 
    The top figure is the input $u(t)$, the second and third figures correspond to each dimension of the output $y(t)$, and the bottom figure is the storage function $V(x(t))-V(x(0))$ (dotted line) and the time integral of supply rate $w(u(t),y(t))$ (solid line).}
    \label{fig:nlink_traj_test}
\end{figure}

\subsection{Flow around cylinder}\label{APP:flow_around_cylinder}
\begin{table}[t]
    \centering
    \begin{tabular}{|c||r|r|r|r|r|}
    \hline 
     Data & Naive  & Stable&  Dissipative \\ \hline 
     \hline  
     N=50    &  $0.154 \pm 0.181$ & ${\bf 0.129} \pm 0.171$ & $0.176 \pm 0.115$ \\  \hline
     N=100    &  ${\bf 0.058} \pm 0.049$ & $0.217 \pm 0.382$ & $0.117 \pm 0.078$ \\  \hline
     N=200    &  $0.098 \pm 0.158$ & ${\bf 0.017} \pm 0.005$ & $0.090 \pm 0.042$ \\ \hline
     \end{tabular}
    \caption{The prediction error (RMSE) of the fluid system benchmark}
    \label{tab:result_fluid_tbl}
\end{table}

\begin{table}[t]
    \centering
    \begin{tabular}{|c||r|r|}
    \hline 
     Data& model& Time(hour) \\ \hline 
     \hline 
     \multirow{3}{*}{
     \begin{tabular}{c}
          N=50
     \end{tabular}}
&Naive	&$4.13	\pm0.92$ \\ \cline{2-3}
&Stable	&$7.47	\pm0.84$ \\ \cline{2-3}
&Dissipative	&$8.99	\pm2.84$ \\ \hline
     \multirow{3}{*}{
     \begin{tabular}{c}
          N=100
     \end{tabular}}
&Naive	&$7.49	\pm1.45$ \\ \cline{2-3}
&Stable	&$13.29	\pm2.17$ \\ \cline{2-3}
&Dissipative	&$15.03	\pm4.52$ \\ \hline
     \multirow{3}{*}{
     \begin{tabular}{c}
          N=200
     \end{tabular}}
&Naive	&$12.53	\pm2.64$ \\ \cline{2-3}
&Stable	&$18.22	\pm7.80$ \\ \cline{2-3}
&Dissipative	&$29.05	\pm11.72$ \\ \hline
     \end{tabular}
    \caption{Computational time of the fluid system benchmark}
    \label{tab:result_flow_time}
\end{table}

\begin{figure}[t]
    \centering
     \includegraphics[width=1.0\linewidth]{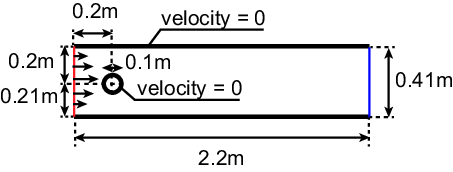}
    \caption{The condition of flow around cylinder.}
    \label{fig:flow_around_cylinder_condition}
\end{figure}

This numerical experiments of fluid simulation were conducted using the following benchmark conditions (\url{https://wwwold.mathematik.tu-dortmund.de/~featflow/en/benchmarks/cfdbenchmarking/flow/dfg_benchmark2_re100.html}).
The fluid simulation was performed using ``dolfinx'' (\url{https://jsdokken.com/dolfinx-tutorial/chapter2/ns_code2.html}).

A fluid with the kinematic viscosity $\nu = 10^{-3}\mathrm{m^2/sec}$ and the fluid density $\rho=1.0\mathrm{kg/m^3}$ was numerically simulated within the flow channel depicted in the Figure~\ref{fig:flow_around_cylinder_condition}.
The inflow velocity on the left boundary was set as a quadratic function along the inlet wall, with its maximum velocity following a triangular wave pattern. The fluid simulation parameters, such as time step intervals and mesh division conditions, were set according to the default parameters specified in dolfinx, i.e., $t \in [0,8], \Delta t = 1/1600$.

\subsubsection{Experimental Setting :}

When learning neural ODEs, the sampling interval $\Delta t$ of the physical simulation is too short, and the learning time is too long, so we thinned it out to 1/20 and learned 640 steps at $t \in [0,1] \Delta t = 1/640$.
When learning neural ODE, the sampling interval $\Delta t$ of the physical simulation is short, training neural networks was performed at $t \in [0,1], \delta t = 1/640$, i.e., 640 time steps in order to shorten the learning time,  

In this benchmark experiment, $V(x)$ must be appropriately designed by the user, but here, we used a function $V(x)=1/2||x||^2$, which is assumed to have one stable point.
In addition, in the dissipative experiment, the following hyperparameters were set:
\begin{align*}
Q = -I_l,\quad S =  0,\quad R =  \gamma^2 I_m
\end{align*}
where $\gamma^2=2$.

Table~\ref{tab:result_fluid_tbl} shows the results of all the experiments corresponding to Figure~\ref{fig:flow_around_cylinder} (B).

In this experiment, 200 triangular wave inputs are created, and $N$ samples are randomly selected to be used as benchmark data.
The computational training time required for one trial of 5000 epochs is listed in Table~\ref{tab:result_flow_time}.
For training each method with neural networks, an NVIDIA Tesla V100 GPU was used.

Related to Table~\ref{tab:result_fluid_tbl},
the difference in the mean values between the naive and other proposed methods was tested using a Bonferroni-corrected t-test with a significance level of 5\%.
As a result, the only improvement that was significant was the stable result for $N=200$.

\begin{table*}[t]
    \centering
    \begin{tabular}{|l|c|c|} \hline
parameter name & range    & type\\ \hline \hline
learning rate & $10^{-5}$  -- $10^{-3}$ & log scale \\ \hline
weight decay  & $10^{-10}$  -- $10^{-6}$ & log scale \\ \hline
batch size    & $ 16   $ -- $128 $ & integer \\ \hline
optimizer     & $\{$ AdamW, Adam, RMSProp$\}$  & categorical\\ \hline
activation     & $\{$ ReLU, LeakyReLU, sigmoid$\}$  & categorical\\ \hline
$\#$layer for $\fn$ & $0$ -- $3$ & integer \\ \hline
$\#$layer for $\gn$ & $0$ -- $3$ & integer\\ \hline
$\#$dim. for a hidden layer of $\fn$ & $8 - 32$ & integer\\ \hline
$\#$dim. for a hidden layer of $\gn$ & $8 - 64$ & integer \\ \hline
Initial scale parameter for $\fn$  & $10^{-5}$ -- $1.0 $ & log scale \\ \hline
$\lambda_2$  & $10^{-10}$ -- $1.0 $ & log scale \\ \hline
    \end{tabular}
    \caption{The search space of Bayesian optimization}
\label{tab:hyper_parameter_Bayesian_optimization}
\end{table*}

\begin{table}[t]
    \centering
    \begin{tabular}{|l|c|} \hline
parameter name & value    \\ \hline \hline
$\lambda_1$  & 0.001 \\ \hline
$\#$layer for $\hn$ & $0$ \\ \hline
$\#$layer for $\eta$ & $0$ \\ \hline
$\#$layer for $l$ & $1$ \\ \hline
$\#$dim. for a hidden layer of $\ln$ & $32$ \\ \hline
$\gamma^2$ & 2.0 \\ \hline
    \end{tabular}
    \caption{Other hyperparameters}
    \label{tab:hyper_parameter_fix}
\end{table}

\section{Neural network architecture and hyper parameters}\label{APP:result_bo}

This section details how to determine the neural network architecture.
In our experiments, 90$\%$ of the dataset is used for training and the remaining 10$\%$ for testing.
To determine hyperparameters, 20$\%$ of the training data is used as validation data.
We ran 100 trials consisting of 10 epochs, selected the hyperparameters that performed best on the validation data and ran 5000 epochs with selected settings.
The architecture and hyperparameters of the neural networks were basically determined by using the tree-structured Parzen estimator (TPE) implemented in Optuna  \cite{optuna_2019}.

Table~\ref{tab:hyper_parameter_Bayesian_optimization} shows the search space of hyperparameters.
The first three parameters: learning rate, weight decay, and batch size are parameters for training the neural networks.
Also, an optimizer is selected from AdamW, Adam, and RMSProp.
The structure of neural network is determined from the number of intermediate layers and dimensions for each layer.
One layer in our setting consists of a fully connected layer with a ReLU activation.
Here, none of the hidden layers corresponds to a linear transformation from input to output.
The last three rows represent parameters related to our proposed methods.
$\lambda_2$ is a hyperparameter of the loss function.
The initial scale parameter is multiplied with the output of $\fn$ to prevent the value of $\fn(x)$ from becoming large in the initial stages of learning.
When $\fn(x)$, which determines the behavior of the internal system, outputs a large value, it diverges due to time evolution, and the learning of the entire system may not progress. 
Therefore, it is empirically preferable to start with a small value for $\fn(x)$ at the initial stage of learning.
The other parameters were fixed as shown in Table~\ref{tab:hyper_parameter_fix}.

\section{Pseudo-code of the learning process}\label{APP4}

Algorithm \ref{alg:a1} shows the pseudo-code of the learning process.
The first line defines the projected dynamics $(\fm,\gm,\hmm)$ from the pre-projected dynamics $(\fs,\gs,\hs)$, defined by the neural network, where $\phi$ is a set of parameters of the pre-projected dynamics.
Note that since the projection is differentiable, the gradients from the projected dynamics can be used to compute the gradients of the pre-projection dynamics by using automatic differentiation.
The 2-7 line represents a training loop, where the gradient-based optimization methods can be used by using the forward and backward calculation.
Note that an ODE solver is used for forward calculation,
and Algorithm \ref{alg:a2} shows the forward calculation when the Euler method is used.
For simplicity, mini-batch computation is omitted in this schematic.

 \begin{algorithm}
 \caption{Training process}
 \label{alg:a1}
 \begin{algorithmic}[1]
 \renewcommand{\algorithmicrequire}{\textbf{Input:}}
 \renewcommand{\algorithmicensure}{\textbf{Output:}}
 \REQUIRE $x_0$: initial state, $u$: input signal,$y$: output signal,   $(\fs,\gs,\hs)$: nominal dynamics, $V$: a designed function
  \STATE define modified functions $(\fm,\gm,\hmm)$ from $(\fs,\gs,\hs) $ and $V$
  \FOR {$1$ to $\#$iterations}
    \STATE $\hat{y} \leftarrow$ ODE with $(\fm,\gm,\hmm)$ from $x_0, u$ ({\bf Algorithm~2})
    \STATE forward computation of ${\rm Loss}$ function (\ref{Eq:Loss}) from $y$
    \STATE $\nabla_\phi {\rm Loss} \leftarrow $ backward computation with ${\rm Loss}$
    \STATE $\phi \leftarrow$ Optimizer($\phi$, $\nabla_\phi {\rm Loss}$)
  \ENDFOR
 \end{algorithmic} 
 \end{algorithm}

\begin{algorithm}
 \caption{Forward computation for dynamics Eq. (1)}
 \label{alg:a2}
 \begin{algorithmic}[1]
 \renewcommand{\algorithmicrequire}{\textbf{Input:}}
 \renewcommand{\algorithmicensure}{\textbf{Output:}}
 \REQUIRE $x_0$: initial state, $u$: input signal, $(\fm,\gm,\hmm)$: dynamics 
 \ENSURE  $\hat{y}$: output signal \\
  \FOR {t $\leftarrow 0$ to $T$}
    \STATE $x_{t+1} \leftarrow x_t + \Delta t (\fm(x_t)+\gm(x_t)u_t)$
    \STATE $\hat{y}_{t} \leftarrow \hmm(x_{t})$
  \ENDFOR
 \RETURN $\hat{y}$ 
 \end{algorithmic} 
\end{algorithm}

\end{document}